\documentclass[a4paper, 11pt]{article} 

\usepackage{latexsym,mathrsfs}
\usepackage{amsmath,amssymb}
\usepackage{amsthm,enumerate,verbatim}
\usepackage{amsfonts}
\usepackage{graphicx}
\usepackage{algorithm}
\usepackage{algorithmic}
\usepackage{url}
\usepackage{paralist}
\usepackage[T1]{fontenc}

\renewcommand{\algorithmicrequire}{\textbf{Input:}}
\renewcommand{\algorithmicensure}{\textbf{Output:}}

\newtheorem{lemma}{Lemma}

\newtheorem{theorem}{Theorem}

\newtheorem{example}{Example}

\newtheorem{remark}{Remark}

\DeclareMathOperator{\conv}{conv} 
\DeclareMathOperator{\argmax}{argmax} 
\DeclareMathOperator{\argmin}{argmin} 
\DeclareMathOperator{\cone}{cone} 
\DeclareMathOperator{\rank}{rank}

\DeclareMathOperator{\diag}{diag}  
\DeclareMathOperator{\tr}{tr}  

\providecommand{\norm}[1]{\lVert#1\rVert}

\title{Robust Near-Separable Nonnegative Matrix Factorization \\ Using Linear Optimization} 

\date{}

\author{Nicolas Gillis\thanks{This work was carried out when NG was a postdoctoral researcher of the fonds de la recherche scientifique (F.R.S.-FNRS).} \\ 
Department of Mathematics and Operational Research\\ 
 Facult\'e Polytechnique, Universit\'e de Mons \\
Rue de Houdain 9, 7000 Mons, Belgium\\
Email: nicolas.gillis@umons.ac.be 
 \and 
 Robert Luce\thanks{RL is supported by Deutsche Forschungsgemeinschaft, Cluster of Excellence ``UniCat''.} \\ 
Institut f\"ur Mathematik, MA 3-3\\ 
Technische Universit\"at Berlin \\
Stra{\ss}e des 17. Juni 136 - 10623 Berlin\\
Email: luce@math.tu-berlin.de}

\begin{document}

\maketitle

\begin{abstract} 
Nonnegative matrix factorization (NMF) has been shown recently to be tractable under the \emph{separability assumption}, under which all the columns of the input data matrix belong to the convex cone generated by only a few of these columns. Bittorf, Recht, R\'e and Tropp (`Factoring nonnegative matrices with linear programs', NIPS 2012) proposed a linear programming (LP) model, referred to as Hottopixx, which is robust under any small perturbation of the input matrix. However, Hottopixx has two important drawbacks: (i) the input matrix has to be normalized, and (ii) the factorization rank has to be known in advance. In this paper, we generalize Hottopixx in order to resolve these two drawbacks, that is, we propose a new LP model which does not require normalization and detects the factorization rank automatically. Moreover, the new LP model is more flexible, significantly more tolerant to noise, and can easily be adapted to handle outliers and other noise models. Finally, we show on several synthetic datasets that it outperforms Hottopixx while competing favorably with two state-of-the-art methods. 
\end{abstract} 

\textbf{Keywords.} Nonnegative matrix factorization, separability, linear programming, convex optimization, robustness to noise, pure-pixel assumption, hyperspectral unmixing.

\section{Introduction}

Nonnegative matrix factorization (NMF) is a powerful dimensionality reduction technique as it automatically extracts sparse and meaningful features from a set of nonnegative data vectors: 
Given $n$ nonnegative $m$-dimensional vectors gathered in a nonnegative matrix $M \in \mathbb{R}^{m \times n}_+$ and a factorization rank $r$, NMF computes two nonnegative matrices  $W \in \mathbb{R}^{m \times r}_+$ and $H \in \mathbb{R}^{r \times n}_+$ such that $M \approx WH$. 
In this way, the columns of the matrix $W$ form a basis for the columns of $M$ since $M(:,j) \approx \sum_{k=1}^r W(:,k) H(k,j)$ for all $j$. 
Moreover, the nonnegativity constraint on the matrices $W$ and $H$ leads these basis elements to represent common localized features appearing in the data set as no cancellation can happen in the reconstruction of the original data. 
Unfortunately, NMF is NP-hard in general \cite{V09}, and highly ill-posed; see \cite{G12} and the references therein. However, if the input data matrix $M$ is \emph{$r$-separable}, that is, if it can be written as 
\[
M = W \, [I_r, \, H'] \Pi, 
\]
where $I_r$ is the $r$-by-$r$ identity matrix, $H' \geq 0$ and $\Pi$ is a permutation matrix, then the problem can be solved in polynomial time, even if some noise is added to the separable matrix $M$ \cite{AGKM11}. Algebraically, separability means that there exists a rank-$r$ NMF $(W,H) \geq 0$ of~$M$ where each column of $W$ is equal to some column of $M$. Geometrically, $r$-separability means that the cone generated by the columns of $M$ has $r$ extreme rays given by the columns of $W$. Equivalently, if the columns of $M$ are normalized so that their entries to sum to one, $r$-separability means that the convex hull generated by the columns of $M$ has $r$ vertices given by the columns of $W$; see, e.g., \cite{KSK12}. 
The separability assumption is far from being artificial in several applications: 
\begin{itemize}

\item In text mining, where each column of $M$ corresponds to a word, separability means that, for each topic, there exists a word associated only with that topic; see \cite{AGKM11, AGM12}. 

\item In hyperspectral imaging, where each column of $M$ equals the spectral signature of a pixel, separability means that, for each constitutive material (``endmember'') present in the image, there exists a pixel containing only that material.  
This assumption is referred to as the \emph{pure-pixel assumption}, and is in general satisfied for high-resolution hyperspectral images; see \cite{Jose12} and the references therein. 

\item In blind source separation, where each column of $M$ is a signal measure at a given point in time, separability means that, for each source, there exists a point in time where only that source is active; see \cite{CMCW08, CMCW11} and the references therein. 

\end{itemize}

Under the separability assumption, NMF reduces to identifying, among the columns of $M$, the columns of $W$ allowing to reconstruct all columns of $M$. In fact, given $W$, the matrix $H$ can be obtained by solving a convex optimization problem $\min_{H \geq 0} \norm{M-WH}$. 

In this paper, we consider the noisy variant of this problem, referred to as \emph{near-separable NMF}: 
\begin{quote}
(Near-Separable NMF) \emph{Given a noisy $r$-separable matrix $\tilde{M} = M + N$ with $M = WH = W[I_r, H'] \Pi$ where $W$ and $H'$ are nonnegative matrices, $\Pi$ is a permutation matrix and $N$ is the noise, find a set $\mathcal{K}$ of $r$ indices such that $\tilde{M}(:,\mathcal{K}) \approx W$. }
\end{quote} 
Several algorithms have been proposed to solve this problem \cite{AGKM11, AGM12, BRRT12, ESV1294, EMO12, GV12, KSK12}. In this paper, our focus is on the linear programming (LP) model proposed by Bittorf, Recht, R\'e and Tropp~\cite{BRRT12} and referred to as Hottopixx. It is described in the next section.

\begin{remark}[Nonnegativity of $\tilde{M}$]  
In the formulation of near-separable NMF, the input data matrix $\tilde{M}$ is not necessarily nonnegative since there is no restriction on the noise $N$. In fact, we will only need to assume that the noise is bounded, but otherwise it is arbitrary; see Section~\ref{srank}. 
\end{remark}

\paragraph{Notation}

Let $A \in \mathbb{R}^{m \times n}$ be a matrix and $x \in \mathbb{R}^m$ a
vector.  We use Matlab-style notation for indexing, for example, $A(i,j)$
denotes the entry of $A$ in the $i$-th row and $j$-th column, while $A(:,j) \in
\mathbb{R}^m$ denotes the $j$-th column of $A$. We use the following notation for
various norms:
\begin{align*}
\norm{x}_1 & = \sum_{i=1}^m |x(i)|,  & 
\norm{A}_1 & = \max_{\norm{x}_1 \le 1} \norm{Ax}_1 = \max_{j} \norm{A(:,j)}_1,  \\
\norm{A}_s & = \sum_{i=1}^m \sum_{j=1}^n |A(i,j)|,  & 
\norm{A}_F & = \sqrt{\sum_{i=1}^m \sum_{j=1}^n A(i,j)^2}. 
\end{align*}

\subsection{Hottopixx, a Linear Programming Model for Near-Separable NMF}

A matrix $M$ is $r$-separable if and only if 
\begin{align} 
M & = WH = W [I_r, H'] \Pi =  [W, WH'] \Pi \nonumber  \\ 
& = [W, WH'] \Pi \, \underbrace{\Pi^{-1} \left(  \begin{array}{cc} 
I_{r} & H'  \\ 
0_{(n-r) \times r} & 0_{(n-r) \times (n-r)} \end{array} \right) \Pi}_{X^0 \in \mathbb{R}^{n \times n}_+}  = MX^0 , \label{sepM}  
\end{align}
for some permutation $\Pi$ and some matrices $W, H' \geq 0$. The matrix $X^0$ is a $n$-by-$n$ nonnegative matrix with $(n-r)$ zero rows such that $M = MX^0$. Assuming the entries of each column of $M$ sum to one, the entries of each column of $W$ and $H'$ have sum to one as well. 
Based on these observations, \cite{BRRT12} proposed to solve the following optimization problem in order to approximately identify the columns of the matrix $W$ among the columns of the matrix $\tilde{M} = M + N$ where $N$ is the noise with \mbox{$\norm{N}_1 \leq \epsilon$} : 
\begin{align} 
\min_{X \in \mathbb{R}^{n \times n}_+} & \quad p^T \diag(X)  \nonumber \\ 
\text{ such that } & \quad \norm{\tilde{M}-\tilde{M}X}_{1} \leq 2 \epsilon , \nonumber\\ 
&  \quad \tr(X) = r,  \label{rechtLP}\\  
&  \quad X(i,i) \leq 1  \text{ for all } i ,  \nonumber \\
&  \quad X(i,j) \leq X(i,i) \text{ for all } i,j , \nonumber 
\end{align}
where $p$ is any $n$-dimensional vector with distinct entries; see Algorithm~\ref{balgo} (in \cite{BRRT12}, the authors use the notation $\norm{\cdot}_{\infty, 1}$ for what we denote by $\norm{\cdot}_1$).
\algsetup{indent=2em}
\begin{algorithm}[ht!]
\caption{Hottopixx - Extracting Columns of a Noisy Separable Matrix using Linear Optimization \cite{BRRT12}} \label{balgo}
\begin{algorithmic}[1] 
    \REQUIRE A normalized noisy $r$-separable matrix $\tilde{M} = WH + N \in \mathbb{R}^{m \times n}_+$, the factorization rank $r$, the noise level $\norm{N}_1 \leq \epsilon$ and a vector $p \in \mathbb{R}^n$ with distinct entries. 
\ENSURE A matrix $\tilde{W}$ such that $\tilde{W} \approx {W}$ (up to permutation).  \medskip 
\STATE Find the optimal solution $X^*$ of \eqref{rechtLP}. 
\STATE Let $\mathcal{K}$ be the index set corresponding to the $r$ largest diagonal entries of $X^*$. 
\STATE Set $\tilde{W} = \tilde{M}(:,\mathcal{K})$.  
\end{algorithmic} 
\end{algorithm} 

Intuitively, the LP model\footnote{Strictly speaking, \eqref{rechtLP} is not a linear program but it can  be reformulated as one.}~\eqref{rechtLP} assigns a total weight $r$ to the $n$ diagonal entries of the variable $X$ in such a way that $\tilde{M}$ can be well approximated using nonnegative linear combinations of columns of $\tilde{M}$ corresponding to positive diagonal entries of $X$. Moreover, the weights used in the linear combinations cannot exceed the diagonal entries of $X$ since $X(:,j) \leq \diag(X)$ for all $j$. 
There are several drawbacks in using the LP model~\eqref{rechtLP} in practice: 
\begin{enumerate}

\item The factorization rank $r$ has to be chosen in advance.  In practice the true factorization rank is often unknown, and a ``good'' factorization rank for the application at hand is typically found by trial and error.  Therefore the LP above may have to be resolved many times.

\item The columns of the input data matrix have to be normalized in order for their entries to sum to one.  This may introduce significant distortions in the data set and lead to poor performance; see~\cite{KSK12} where some numerical experiments are presented.

\item The noise level $\norm{N}_{1} \leq \epsilon$ has to be estimated. 

\item One has to solve a rather large optimization problem with $n^2$ variables, so that the model cannot be used directly for huge-scale problems. 
\end{enumerate}

It is important to notice that there is no way to getting rid of both drawbacks 2.\@ and 3. In fact, in the noisy case, the user has to indicate either 
\begin{itemize}

\item  The factorization rank $r$, and the algorithm should find a subset of $r$ columns of $\tilde{M}$ as close as possible to the columns of $W$, or 

\item  The noise level $\epsilon$, and the algorithm should try to find the smallest possible subset of columns of $\tilde{M}$ allowing to approximate $\tilde{M}$ up to the required accuracy. 

\end{itemize}

\subsection{Contribution and Outline of the Paper}

In this paper, we generalize Hottopixx in order to resolve drawbacks 1.\@ and 2.\@ above. More precisely, we propose a new LP model which has the following properties: 
\begin{itemize}

\item Given the noise level $\epsilon$, it detects the number $r$ of columns of $W$ automatically; see Section~\ref{srank}.

\item It can be adapted to dealing with outliers; see Section~\ref{outliers}. 

\item It does not require column normalization; see Section~\ref{colnorm}. 

\item It is significantly more tolerant to noise than Hottopixx. 
In fact, we propose a tight robustness analysis of the new LP model proving its superiority (see~Theorems~\ref{th1}~and~\ref{th1b}). 
This is illustrated in Section~\ref{xp} on several synthetic data sets, where the new LP model is shown to outperform Hottopixx while competing favorably with two  state-of-the-art methods, namely the successive projection algorithm (SPA) \cite{MC01, GV12} and the fast conical hull algorithm (XRAY) \cite{KSK12}. 

\end{itemize} 
The emphasis of our work lies in a thorough theoretical understanding of such
LP based approaches, and the numerical experiments in Section \ref{xp}
illustrate the proven robustness properties.  An implementation for
real-word, large-scale problems is, however, a topic outside the scope of this
work (see Section~\ref{future}).

  \section{Detecting the Factorization Rank Automatically} \label{srank}

  In this section, we analyze the following LP model: 
  \begin{align} 
\min_{X \in \mathbb{R}^{n \times n}_+} & \quad p^T \diag(X)  \nonumber \\ 
\text{ such that } & \quad \norm{\tilde{M}-\tilde{M}X}_{1} \leq \rho \epsilon,  \label{gillisLP} \\ 
&  \quad X(i,i) \leq 1  \text{ for all } i ,    \nonumber \\
&  \quad X(i,j) \leq X(i,i) \text{ for all } i,j ,  \nonumber 
\end{align} 
where $p$ has \emph{positive} entries and $\rho > 0$ is a parameter.  We also analyze the corresponding near-separable NMF algorithm (Algorithm~\ref{hot2}) with an emphasis on \emph{robustness}. 
\algsetup{indent=2em}
\begin{algorithm}[ht!]
\caption{Extracting Columns of a Noisy Separable Matrix using Linear Optimization} \label{hot2} 
\begin{algorithmic}[1] 
\REQUIRE A normalized noisy $r$-separable matrix $\tilde{M} = WH + N \in \mathbb{R}^{m \times n}_+$, the noise level $\norm{N}_1 \leq \epsilon$, a parameter $\rho > 0$ and a vector $p \in \mathbb{R}^n$ with positive distinct entries. 
\ENSURE An $m$-by-$r$ matrix $\tilde{W}$ such that $\tilde{W} \approx {W}$ (up to permutation). \medskip 
\STATE Compute an optimal solution $X^*$ of \eqref{gillisLP}. 
\STATE Let $\mathcal{K}$ be the index set corresponding to the diagonal entries of $X^*$ larger than $1 - \frac{\min(1,\rho)}{2}$. 
\STATE $\tilde{W} = \tilde{M}(:,\mathcal{K})$. 
\end{algorithmic} 
\end{algorithm} 
  The LP model~\eqref{gillisLP} is exactly the same as \eqref{rechtLP} except that  the constraint $\tr(X) = r$ has been removed, and that there is an additional parameter $\rho$.   
  Moreover, the vector $p \in \mathbb{R}^{n}$ in the objective function has to be \emph{positive}, or otherwise any diagonal entry of an optimal solution of \eqref{gillisLP} corresponding to a negative entry of $p$ will be equal to one (in fact, this reduces the objective function the most while minimizing $\norm{M-MX}_{1}$). 
A natural value for the parameter $\rho$ is two, as in the original LP model~\eqref{rechtLP}, so that the matrix $X^0$ in Equation~\eqref{sepM} identifying the set of columns of $\tilde{M}$ corresponding to the columns of $W$ is feasible. However, the model~\eqref{gillisLP} is feasible for any $\rho \geq 0$ since the identity matrix of dimension $n$ (that is, $X = I_n$) is always feasible. Hence, it is not clear a priori which value of $\rho$ should be chosen. 
The reason we analyze the LP model~\eqref{gillisLP} for different values of $\rho$ is two-fold: 

\begin{itemize}

\item First, it shows that the LP model  \eqref{gillisLP} is rather flexible as it is not too sensitive to the right-hand side of the constraint $\norm{M-MX}_{1} \leq \rho \epsilon$. In other terms, 
 the noise level does not need to be known precisely for the model to make sense. 
 This is a rather desirable property as, in practice, the value of $\epsilon$ is typically only known/evaluated approximately. 
 
\item Second, we observed that taking $\rho$ smaller than two gives in average significantly better results (see Section~\ref{xp} for the numerical experiments). Our robustness analysis of Algorithm~\ref{hot2} will suggest that the best choice is to take $\rho = 1$. 

\end{itemize}

In this section, we prove that the LP model~\eqref{gillisLP} allows to identifying approximately the columns of the matrix $W$ among the columns of the matrix $\tilde{M}$ \emph{for any  $\rho > 0$}, given that the noise level $\epsilon$ is sufficiently small ($\epsilon$ will depend on the value $\rho$); see Theorems~\ref{th1},~\ref{th1b}~and~\ref{th2}.

Before stating the robustness results, let us define the conditioning of a nonnegative matrix $W$ whose entries of each column sum to one: 
\[
\kappa = \min_{1 \leq k \leq r} \min_{x \in \mathbb{R}^{r-1}_+} \norm{W(:,k) - W(:,\mathcal{K})x}_1, \quad \text{ where } \mathcal{K} = \{1,2,\dots,r\} \backslash \{k\}, 
\]
and the matrix $W$ is said to be \emph{$\kappa$-robustly conical}. The parameter $0 \leq \kappa \leq 1$ tells us how well the columns of $W$ are spread in the unit simplex. In particular, if $\kappa = 1$, then $W$ contains the identity matrix as a submatrix (all other entries being zeros) while, if $\kappa = 0$, then at least one of the columns of $W$ belongs to the convex cone generated by the others. 
Clearly, the better the columns of $W$ are spread across the unit simplex, the less sensitive is the data to noise. 
For example, $\epsilon < \frac{\kappa}{2}$ is a necessary condition to being able to distinguish the columns of $W$ \cite{G12b}.

\subsection{Robustness Analysis without Duplicates and Near Duplicates} \label{robdup}

In this section, we assume that the columns of $W$ are isolated (that is, there is no duplicate nor near duplicate of the columns of $W$ in the data set) hence more easily identifiable. This type of margin constraint is typical in machine learning \cite{BRRT12}, and is equivalent to bounding the entries of $H'$ in the expression ${M} = W[I_r, H']\Pi$, see Equation~\eqref{sepM}.  
In fact, for any $1 \leq k \leq r$ and $h \in \mathbb{R}^r_+$ with $\max_i h(i) \leq \beta \leq 1$, we have that 
\begin{align*}
\norm{W(:,k) - Wh}_1
& = \norm{(1-h(k)) W(:,k) - W(:,\mathcal{K})h(\mathcal{K})}_1 \\
& \geq (1-\beta) \min_{y \in \mathbb{R}^{r-1}_+} \norm{W(:,k) - W(:,\mathcal{K})y}_1 \\ 
& \geq (1-\beta) \kappa, 
\end{align*}
where $\mathcal{K} = \{1,2,\dots,r\} \backslash \{k\}$. Hence $\max_{ij} H'_{ij} \leq \beta$ implies that all data points are at distance at least $(1-\beta) \kappa$ of any column of~$W$. 
Under this condition, we have the following robustness result: 
\begin{theorem} \label{th1} 
Suppose $\tilde{M} = M+N$ where the entries of each column of $M$ sum to one, $M = WH$ admits a rank-$r$ separable factorization of the form~\eqref{sepM} with 
\mbox{$\max_{ij} H'_{ij} \leq \beta \leq 1$} and $W$ $\kappa$-robustly conical with $\kappa > 0$, and $\norm{N}_{1} \leq \epsilon$. If  
\[
\epsilon \leq \frac{\kappa (1-\beta) \min(1,\rho)}{5(\rho+2)}, 
\] 
then Algorithm~\ref{hot2} extracts a matrix $\tilde{W} \in \mathbb{R}^{m \times r}$  satisfying $\norm{W - \tilde{W}(:,P)}_{1} \leq \epsilon$ for some permutation~$P$. 
\end{theorem}
\begin{proof}
See Appendix~\ref{appA}. 
\end{proof}
\begin{remark}[Noiseless case] When there is no noise (that is, $N=0$ and $\epsilon = 0$), duplicates and near duplicates are allowed in the data set; otherwise $\epsilon > 0$ implying that $\beta < 1$ hence the columns of $W$ are isolated. 
\end{remark}

\begin{remark}[A slightly better bound]
The bound on the allowable noise in Theorem \ref{th1} can be slightly improved, so that under the same conditions we can allow a noise level of
\begin{equation*}
    \epsilon < \frac{\kappa (1-\beta) \min(1,\rho)}{4(\rho + 2)
        + \kappa (1-\beta) \min(1,\rho)}.
\end{equation*}
However, the scope for substantial improvements is limited, as we will show in
Theorem \ref{th1b}.
\end{remark}

\begin{remark}[Best choice for $\rho$] \label{rem2}
Our analysis suggests that the best value for $\rho$ is one. In fact, 
\[
\argmax_{\rho \geq 0} \frac{\min(1,\rho)}{(\rho+2)} \; = \; 1. 
\]
In this particular case, the upper bound on the noise level to guarantee recovery is given by $\epsilon \leq \frac{\kappa (1-\beta)}{15}$ while, for $\rho = 2$, we have $\epsilon \leq \frac{\kappa (1-\beta)}{20}$. The choice $\rho=1$ is also optimal in the same sense for the bound in the previous remark.
We will see in Section~\ref{xp}, where we present some numerical experiments, that choosing $\rho = 1$ works remarkably better than $\rho = 2$. 
\end{remark}

It was proven by \cite{G12b} that, for Algorithm~\ref{balgo} to extract the columns of $W$ under the same assumptions as in Theorem~\ref{th1}, it is necessary that 
  \[
  \epsilon < \frac{\kappa (1-\beta)}{(r-1)(1-\beta)+1} \quad \text{ for any $r \geq 3$ and $\beta < 1$,  }
  \] 
  while it is sufficient that $\epsilon \leq \frac{\kappa (1-\beta)}{9(r+1)}$.  Therefore, 
   if there are no duplicate nor near duplicate of the columns of $W$ in the data set,  
 \begin{quote}
 \emph{Algorithm~\ref{hot2} is more robust than Hottopixx (Algorithm~\ref{balgo}): in fact, unlike Hottopixx,  its bound on the noise to guarantee recovery (up to the noise level) is \emph{independent of the number of columns of $W$}.  Moreover, given the noise level, it detects the number of columns of $W$ automatically.} 
  \end{quote} 
  
The reason for the better performance of Algorithm~\ref{hot2} is the following: for most noisy $r$-separable matrices $\tilde{M}$, there typically exist matrices $X'$ satisfying the constraints of \eqref{gillisLP} and  such that $\tr(X') < r$. 
Therefore, the remaining weight $\left(r-\tr(X')\right)$ will be assigned by Hottopixx to the diagonal entries of $X'$ corresponding to the smallest entries of $p$, since the objective is to minimize $p^T \diag(X')$. 
These entries are unlikely to correspond to columns of $W$ 
(in particular, if $p$ in chosen by an adversary). 
We observed that when the noise level $\epsilon$ increases, $r-\tr(X')$ increases as well, hence it becomes likely that some columns of $W$ will not be identified. 
\begin{example} \label{ex1}
Let us consider the following simple instance: 
\[
M = \underbrace{I_r}_{= W} \underbrace{\left[I_r,  \frac{e}{r}\right]}_{= H} \in \mathbb{R}^{r \times (r+1)} 
\quad \text{ and } \quad N = 0,  
\]
where $e$ is the vector of all ones. We have that $||N||_1 = 0 \leq \epsilon$ for any $\epsilon \geq 0$. 

Using $p = [1,2,\dots,r,-1]$ in the objective function, the Hottopixx LP~\eqref{rechtLP} will try to put as much weight as possible on the last diagonal entry of $X$ (that is, $X(r+1,r+1)$) which corresponds to the last column of $M$. 
Moreover, because $W$ is the identity matrix, no column of $W$ can be used to reconstruct another column of $W$ (this could only increase the error) so that Hottopixx has to assign a weight to the first $r$ diagonal entries of $X$ larger than $(1-2\epsilon)$ (in order for the constraint $||M-MX||_1 \leq 2 \epsilon$ to be satisfied). 
The remaining weight of $2 r \epsilon$ (the total weight has to be equal to $r$) can be assigned to the last column of $M$. Hence, for $1-2\epsilon < 2 r \epsilon \iff \epsilon > \frac{1}{2(r+1)}$, Hottopixx will fail as it will extract the last column of $M$. 

Let us consider the new LP model~\eqref{gillisLP} with $\rho = 2$. For the same reason as above, it has to assign a weight to the first $r$ diagonal entries of $X$ larger than $(1-2 \epsilon)$. Because the cost of the last column of $M$ has to be positive (that is, $p(r+1) > 0$), the new LP model~\eqref{gillisLP} will try to minimize the last diagonal entry of $X$ (that is, $X(r+1,r+1)$). Since $M(:,r+1) = \frac{1}{r}We$, $X(r+1,r+1)$ can be taken equal to zero taking $X(1:r,r+1) = 1-2\epsilon$. Therefore, for any positive vector $p$, any $r$ and any $\epsilon < \frac{1}{2}$, the new LP model~\eqref{gillisLP} will identify correctly all columns of $W$. (For other values of $\rho$, this will be true for any $\epsilon < \frac{1}{\rho}$.) 
\end{example}

This explains why the LP model enforcing the constraint $\tr(X) = r$ is less robust, and why its bound on the noise depends on the factorization rank $r$. Moreover, the LP \eqref{rechtLP} is also much more sensitive to the parameter~$\epsilon$ than the model LP \eqref{gillisLP}: 
\begin{itemize}
\item For $\epsilon$ sufficiently small, it becomes infeasible, while, 
\item for $\epsilon$ too large, the problem described above is worsened: there are matrices $X'$ satisfying the constraints of \eqref{gillisLP} and  such that $\tr(X') \ll r$, hence Hottopixx will perform rather poorly (especially in the worst-case scenario, that is, if the problem is set up by an adversary). 
\end{itemize}

  To conclude this section, we prove that the bound on the noise level $\epsilon$ to guarantee the recovery of the columns of $W$ by Algorithm~\ref{hot2} given in Theorem~\ref{th1} is tight up to some constant multiplicative factor. 
  \begin{theorem} \label{th1b} 
  For any fixed $\rho > 0$ and $\beta < 1$, the bound on $\epsilon$ in Theorem~\ref{th1} is tight up to a multiplicative factor. In fact, under the same assumptions on the input matrix $\tilde{M}$, it is necessary   that 
  $\epsilon < \frac{\kappa (1-\beta) \min(1,\rho)}{2 \rho}$  
  for Algorithm~\ref{hot2} to extract a matrix $\tilde{W} \in \mathbb{R}^{m \times r}$  satisfying $\norm{W - \tilde{W}(:,P)}_{1} \leq \epsilon$ for some permutation~$P$.
\end{theorem}
\begin{proof}
See Appendix~\ref{appAa}. 
\end{proof} 
For example,   Theorem~\ref{th1b} implies that, for $\rho = 1$, the bound of Theorem~\ref{th1} is tight up to a factor $\frac{15}{2}$.

  \subsection{Robustness Analysis with Duplicates and Near Duplicates} \label{rob}

  In case there are duplicates and near duplicates in the data set, it is necessary to apply a post-processing to the solution of \eqref{gillisLP}. 
  In fact, although we can guarantee that there is a subset of the columns of $\tilde{M}$ close to each column of $W$ whose sum of the corresponding diagonal entries of an optimal solution of~\eqref{gillisLP} is large,   there is no guarantee that the weight will be concentrated only in one entry.   
  It is then required to apply some post-processing based on the distances between the data points to the solution of~\eqref{gillisLP} (instead of simply picking the $r$ indices corresponding to its largest diagonal entries) in order to obtain a robust algorithm. 
  In particular, using Algorithm~\ref{cluster} to post-process the solution of~\eqref{rechtLP} leads to a more robust algorithm than Hottopixx \cite{G12b}.  Note that pre-processing would also be possible \cite{EMO12, AGKM11}.

 Therefore, we propose to post-process an optimal solution of \eqref{gillisLP} with Algorithm~\ref{cluster}; see Algorithm~\ref{postpro}, for which we can prove the following robustness result: 
   \begin{theorem} \label{th2}
  Let $M=WH$  be an $r$-separable matrix whose entries of each column sum to one and of the form~\eqref{sepM} with $H \geq 0$ and $W$ $\kappa$-robustly conical. 
  Let also $\tilde{M} = M+N$  with $\norm{N}_{1} \leq \epsilon$. If 
   \[ 
  \epsilon < \frac{\omega \kappa}{99(r+1)}, 
  \]  
  where $\omega = \min_{i \neq j} \norm{W(:,i)-W(:,j)}_1$,    
  then Algorithm~\ref{postpro} extracts a matrix $\tilde{W}$ such that 
  \[
  \norm{W - \tilde{W}(:,P)}_{1} \leq 49(r+1) \frac{\epsilon}{\kappa} + 2 \epsilon, \quad \text{   for some permutation $P$.} 
  \]
  \end{theorem}
\begin{proof}
See Appendix~\ref{appB} (for simplicity, we only consider the case $\rho = 2$; the proof can be generalized for other values of $\rho > 0$ in a similar way as in Theorem~\ref{th1}). 
\end{proof}

This robustness result follows directly from \cite[Theorem 5]{G12b}, and is the same as for the algorithm using the optimal solution of \eqref{rechtLP} post-processed with Algorithm~\ref{cluster}. 
Hence, in case there are duplicates and near duplicates in the data set, we do not know if Algorithm~\ref{postpro} is more robust, although we believe the bound for Algorithm~\ref{postpro} can be improved (in particular, that the dependence in $r$ can be removed), this is a topic for further research.

\algsetup{indent=2em}
\begin{algorithm}[ht!]
\caption{Extracting Columns of a Noisy Separable Matrix using Linear Optimization} \label{postpro}
\begin{algorithmic}[1] 
\REQUIRE A normalized $r$-separable matrix $\tilde{M} = WH + N$, and the noise level   
$\norm{N}_{1} \leq \epsilon$. 
\ENSURE An $m$-by-$r$ matrix $\tilde{W}$ such that $\tilde{W} \approx {W}$ (up to permutation).  \medskip 
\STATE Compute the optimal solution $X^*$ of \eqref{gillisLP} where $p = e$ is the vector of all ones and $\rho = 2$. 
\STATE $K$ = post-processing$\left( \tilde{M},\diag(X^*),\epsilon \right)$; 
\STATE $\tilde{W} = \tilde{M}(:,\mathcal{K})$; 
\end{algorithmic} 
\end{algorithm}  
\algsetup{indent=2em}
\begin{algorithm}[ht!]
\caption{Post-Processing - Clustering Diagonal Entries of $X^*$ \cite{G12b} \label{cluster}}
\begin{algorithmic}[1] 
\REQUIRE A matrix $\tilde{M} \in \mathbb{R}^{m \times n}$, a vector $x \in \mathbb{R}^{n}_+$, $\epsilon \geq 0$, and possibly a factorization rank $r$. 
\ENSURE A index set $\mathcal{K}^*$ with $r$ indices so that the columns of $\tilde{M}(:,\mathcal{K}^*)$ are centroids whose corresponding clusters have large weight (the weights of the data points are given~by~$x$). \medskip 
\STATE  $D(i,j) = \norm{m_i-m_j}_1$ for $1 \leq i, j \leq n$; 
\IF{ $r$ is not part of the input }
\STATE $r = \Big\lceil  \sum_i x(i) \Big\rceil$;   
\ELSE \STATE $x \leftarrow r \frac{x}{\sum_i x(i)}$; 
\ENDIF
\STATE $\mathcal{K} = \mathcal{K}^* = \left\{ k \ | \ x(k) > \frac{r}{r+1} \right\}$ 
and $\nu = \nu^* = \max\left( 2 \epsilon, \min_{\{(i,j) | D(i,j) > 0\}} D(i,j)  \right)$;  \vspace{0.1cm}  
\WHILE{$|\mathcal{K}| < r$ and $\nu < \max_{i,j} D(i,j)$} 
\STATE  $\mathcal{S}_i = \{ j \ | \ D(i,j) \leq \nu \}$ for $1 \leq i \leq n$; 
\STATE  $w(i) = \sum_{j \in \mathcal{S}_i} x(j)$ for $1 \leq i \leq n$; 
\STATE $\mathcal{K} = \emptyset$; 
\WHILE{$\max_{1 \leq i \leq n} w(i) > \frac{r}{r+1}$} 
\STATE $k = \argmax w(i)$; $\mathcal{K} \leftarrow \mathcal{K} \cup \{ k \}$; 
\STATE For all $1 \leq i \leq n$ and $j \in \mathcal{S}_k \cup \mathcal{S}_i$ : $w(i) \leftarrow w(i)-x(j)$; 
\ENDWHILE 
\IF{$|\mathcal{K}| > |\mathcal{K}^*|$}
\STATE $\mathcal{K}^* = \mathcal{K}$; $\nu = \nu^*$;
\ENDIF 
\STATE $\nu \leftarrow 2 \nu$; 
\ENDWHILE 
\STATE \emph{\% Safety procedure in case the conditions of Theorem~\ref{th2} are not satisfied: }
\IF{$|\mathcal{K}^*| < r$} 
\STATE $\text{d} = \max_{i,j} D(i,j)$; 
\STATE  $\mathcal{S}_i = \{ j \ | \ D(i,j) \leq \nu^* \}$ for $1 \leq i \leq n$; 
\STATE  $w(i) = \sum_{j \in \mathcal{S}_i} x(j)$ for $1 \leq i \leq n$; 
\STATE $\mathcal{K}^* = \emptyset$;
\WHILE{$|\mathcal{K}^*| < r$} 
\STATE $k = \argmax w(i)$; $\mathcal{K}^* \leftarrow \mathcal{K}^* \cup \{ k \}$; 
\STATE For all $1 \leq i \leq n$, and $j \in \mathcal{S}_k \cup \mathcal{S}_i$ : 
$w(i) \leftarrow w(i) -  \left( \frac{\text{d}-D(i,j)}{\text{d}} \right)^{0.1}  x(j)$; 
\STATE $w(k) \leftarrow 0$; 
\ENDWHILE  
\ENDIF 
\end{algorithmic} 
\end{algorithm}  

\begin{remark}[Choice of $p$] Although Theorem~\ref{th2} requires the entries of the vector $p$ to be all ones, we recommend to take the entries of $p$ distinct, but close to one. 
 This allows the LP~\eqref{gillisLP} to discriminate better between the duplicates hence Algorithm~\ref{postpro} does not necessarily have to enter the post-processing loop. 
 We suggest to use 
$p(i) \sim 1 + \mathcal{U}(-\sigma,\sigma)$ for all $i$, where 
$\sigma \ll 1$ 
and $\mathcal{U}(a,b)$ is the uniform distribution in the interval $[a,b]$. 
\end{remark}

\section{Handling Outliers} \label{outliers}

Removing the rank constraint has another advantage: it allows to deal with
outliers.  If the data set contains outliers, the corresponding diagonal entries
of an optimal solution $X^*$ of~\eqref{gillisLP} will have to be large (since
outliers cannot be approximated well with convex combinations of points in the
data set).  However, under some reasonable assumptions, outliers are useless to
approximate data points, hence off-diagonal entries of the  rows of $X^*$
corresponding to outliers will be small. Therefore, one could discriminate
between the columns of $W$ and the outliers by looking at the
\emph{off-diagonal entries of~$X^*$}.  
This result is closely related to the one presented in \cite{GV12} (Section 3).
For simplicity, we consider in this section only the case where $\rho = 2$ and
assume absence of duplicates and near-duplicates in the data set; the more
general case can be treated in a similar way. 

Let the columns of $T \in \mathbb{R}^{m \times t}$ be $t$ outliers added to the separable matrix $W[I_r, H']$ along with some noise to obtain
\begin{align}
\tilde{M} = M+N \; \text{ where } \; M &  = [ W, \, T ]  H  
= \left[ W, \, T, \, WH' \right] \Pi 
\left[ \begin{array}{ccc} 
I_r & 0_{r \times t} & H' \\ 
0_{t \times r} & I_t   &  0_{t \times r}   \end{array} \right] \Pi,  \label{modelo}  
\end{align} 
which is a noisy $r$-separable matrix containing $t$ outliers. We propose Algorithm~\ref{hotout} to  approximately extract the columns of $W$ among the columns of $\tilde{M}$. 
\algsetup{indent=2em}
\begin{algorithm}[ht!]
\caption{Extracting Columns of a Noisy Separable Matrix with Outliers using Linear Optimization} \label{hotout} 
\begin{algorithmic}[1] 
\REQUIRE A normalized noisy $r$-separable matrix $\tilde{M} = [W, T, WH']  \Pi + N \in \mathbb{R}^{m \times n}_+$ with outliers, the noise level $\norm{N}_1 \leq \epsilon$ and a vector $p \in \mathbb{R}^n$ with positive distinct entries and  $\rho = 2$. 
\ENSURE An $m$-by-$r$ matrix $\tilde{W}$ such that $\tilde{W} \approx {W}$ (up to permutation). \medskip 
\STATE Compute the optimal solution $X^*$ of \eqref{gillisLP} 
where $p$ has distinct positive entries. 
\STATE Let $\mathcal{K} = \left\{ 1 \leq k \leq n \; | \; X^*(k,k) \geq \frac{1}{2} 
\text{ and } \norm{X^*(k,:)}_1 - X^*(k,k)  \geq \frac{1}{2}  \right\}$. 
\STATE $\tilde{W} = \tilde{M}(:,\mathcal{K})$. 
\end{algorithmic} 
\end{algorithm} 

In order for Algorithm~\ref{hotout} to extract the correct set of columns of $\tilde{M}$, the off-diagonal entries of the rows corresponding to the columns of $T$ (resp.\@ columns of $W$) must be small (resp.\@ large). This can be guaranteed using the following conditions (see also Theorem~\ref{outth} below): 

\begin{itemize}

\item The angle between the cone generated by the columns of $T$ and the columns space of $W$ is positive.  More precisely, we will assume that for all $1 \leq k \leq t$ 
\begin{equation} \label{angle}
\min_{ 
\begin{array}{c} x \in \mathbb{R}^t_+, x(k) = 1, \\ y \in \mathbb{R}^r \end{array} 
} 
\norm{Tx - Wy}_1 \; \geq \; \eta > 0 . 
\end{equation} 
In fact, if a nonnegative linear combination of outliers (that is, $Tx$  with $x \geq 0$)  belongs to the column space of $W$, then some data points can usually be reconstructed using a non-zero weight for these outliers (it suffices that some data points belong to the convex hull of some columns of $W$ and that linear combination of outliers). 

  \item The matrix $[W,T]$ is robustly conical, otherwise some columns of $T$ could be reconstructed using other columns of $T$ whose corresponding rows could hence have large off-diagonal entries.

\item Each column of $W$ is necessary to reconstruct at least one data point, otherwise the off-diagonal entries of the row of $X^*$ corresponding to that `useless' column of $W$ will be small, possibly equal to zero, and it cannot be distinguished from an outlier.  
More formally, for all $1 \leq k \leq r$, there is a least one data point $M(:,j) = WH(:,j) \neq W(:,k)$ such that 
\begin{equation} \label{gc}
\min_{x \geq 0, y \geq 0} \norm{M(:,j) - Tx - W(:,{\mathcal{K}}) y}_1 \geq \delta, \quad \text{where ${\mathcal{K}} = \{1,2,\dots,r\} \backslash \{k\}$}. 
\end{equation} 
If Equation \eqref{angle} holds, this condition is satisfied for example when $\conv(W)$ is a simplex and some points lie inside that simplex (it is actually satisfied if and only if each column of $W$ define with other columns of $W$ a simplex containing at least one data point in its interior).  

\end{itemize}

These conditions allow to distinguish the columns of $W$ from the outliers using off-diagonal entries of an optimal solution $X^*$ of \eqref{gillisLP}: 
\begin{theorem} \label{outth} 
Suppose $\tilde{M} = M+N$ where the entries of each column of $M$ sum to one, $M = [W,T]H$ has the form~\eqref{modelo} with $H \geq 0$, \mbox{$\max_{ij} H'_{ij} \leq \beta \leq 1$} and $[W, T]$ $\kappa$-robustly conical, and $\norm{N}_{1} \leq \epsilon$. 
Suppose also that $M$, $W$ and $T$ satisfy Equations~\eqref{angle} and~\eqref{gc} for some $\eta > 0$ and $\delta > 0$. 
If  
\[
\epsilon \leq \frac{\nu (1-\beta)}{20 (n-1)} \quad \text{ where $\nu = \min(\kappa,\eta,\delta)$}, 
\]
then Algorithm~\ref{hotout} extracts a matrix $\tilde{W} \in \mathbb{R}^{m \times r}$  satisfying $\norm{W - \tilde{W}(:,P)}_{1} \leq \epsilon$ for some permutation~$P$. 
\end{theorem}
\begin{proof}
See Appendix~\ref{appout}. 
\end{proof}

Unfortunately, the factor $\frac{1}{n-1}$ is necessary because a row of $X^*$ corresponding to an outlier 
could potentially be assigned weights proportional to $\epsilon$ for all off-diagonal entries.   
For example, if all data points are perturbed in the direction of an outlier, that is, 
$N(:,j) = \epsilon \, T(:,k)$ for all $j$ and for some $1 \leq k \leq t$, then we could have $\sum_{j \neq k} X(k,j) = (n-1) \mathcal{O}(\epsilon)$ hence it is necessary that $\epsilon \leq \mathcal{O}(n^{-1})$ (although it is not likely to happen in practice). 
A simple way to improve the bound is the following: 
\begin{itemize}

\item Identify the vertices and outliers using  $\mathcal{K} =  \left\{ 1 \leq k \leq n \; | \; X^*(k,k) \geq \frac{1}{2}  \right\}$ (this only requires \mbox{$\epsilon \leq \frac{\kappa (1-\beta)}{20}$}, cf.\@ Theorem~\ref{th1}). 

\item Solve the linear program $Z^* = \argmin_{Z \geq 0} \norm{M-M(:,\mathcal{K})Z}_1$. 

\item Use the sum of the rows of $Z^*$ (instead of $X^*)$ to identify the columns of $W$. 

\end{itemize} 
Following the same steps as in the proof of Theorem~\ref{outth}, the bound for $\epsilon$ for the corresponding algorithm becomes $\epsilon \leq \frac{\nu (1-\beta)}{20 (r+t-1)}$. 

\begin{remark}[Number of outliers] 
Algorithm~\ref{hotout} does not require the number of outliers as an input.  
Moreover, the number of outliers is not limited hence our result is stronger than the one of \cite{GV12} where the number of outliers cannot exceed $m - r$ (because $T$ needs to be full rank, while we only need $T$ to be robustly conical and the cone generated by its columns define a wide angle with the column space of~$W$). 
\end{remark}

\begin{remark}[Hottopixx and outliers] 
Replacing the constraint $\tr(X) = r$ with $\tr(X) = r + t$ ($r$ is the number of columns of $W$ and $t$ is the number of outliers) in the LP model~\eqref{rechtLP} allows to deal with outliers. However, the number of outliers plus the number of columns of $W$ (that is, $r+t$) has to be estimated, which is rather impractical. 
\end{remark}

  \section{Avoiding Column Normalization}  \label{colnorm}
  
  In order to use the LP models \eqref{rechtLP} and \eqref{gillisLP}, normalization must be enforced which may introduce significant distortions in the data set and lead to poor performances \cite{KSK12}.  
	If $M$ is $r$-separable but the entries of each column do not sum to one, we still have that 
  \[
  M = W[I_r, H'] \Pi = [W, WH'] \Pi = [W, WH'] \left( \begin{array}{cc} 
I_{r} & H'  \\ 
0_{(n-r) \times r} & 0_{(n-r) \times (n-r)} \end{array} \right) \Pi = M X^0. 
  \]
  However, 
  the constraints 
  $X(i,j) \leq X(i,i)$ for all $i,j$  
 in the LP's \eqref{rechtLP} and \eqref{gillisLP} are not necessarily satisfied by the matrix $X^0$, because the entries of $H'$ can be arbitrarily large.   
 
  Let us denote $\tilde{M}_o$ the original unnormalized noisy data matrix,  and  its normalized version $\tilde{M}$, with 
  \[
  \tilde{M}(:,j) = \frac{\tilde{M}_o(:,j)}{\norm{\tilde{M}_o(:,j)}_1} \quad \text{ for all } j.  
  \]  
  Let us also rewrite the LP \eqref{gillisLP} in terms of $\tilde{M}_o$ instead of $\tilde{M}$ using the following change of variables 
  \begin{equation} \label{changeVar}
 X_{ij} =   \frac{\norm{\tilde{M}_o(:,i)}_1}{\norm{\tilde{M}_o(:,j)}_1}   Y_{ij}   \quad \text{ for all $i$, $j$} . 
  \end{equation}
  Note that  $Y_{ii} = X_{ii}$ for all $i$. We have for all $j$ that 
  \begin{align*}
  \left\|  \tilde{M}(:,j) - \sum_i \tilde{M}(:,i) X_{ij}\right\|_{1}
  & = \left\| \frac{\tilde{M}_o(:,j)}{\norm{\tilde{M}_o(:,j)}_1}  - \sum_j \frac{\tilde{M}_o(:,i)}{\norm{\tilde{M}_o(:,i)}_1} \frac{\norm{\tilde{M}_o(:,i)}_1}{\norm{\tilde{M}_o(:,j)}_1} Y_{ij} \right\|_{1} \\ 
  & = \frac{1}{\norm{\tilde{M}_o(:,j)}_1} \left\| {\tilde{M}_o(:,j)}  - \sum_j {\tilde{M}_o(:,i)} Y_{ij} \right\|_{1} ,  
    \end{align*}
 which proves that the following LP  
  \begin{align} 
\min_{Y \in \mathcal{Y}} & \quad p'^T \diag(Y)  \quad
\text{ such that } \; {\norm{\tilde{M}_o(:,j) - \tilde{M}_o Y(:,j)}_{1} \leq \rho \epsilon  \norm{\tilde{M}_o(:,j)}_1 \; \text{ for all $j$}}   \label{gillisLP1} , 
\end{align}
where 
\begin{equation} \label{yset}
\mathcal{Y} = 
\{ Y \in \mathbb{R}^{n \times n}_+ 
\ | \ 
Y(i,i) \leq 1  \;  \forall \, i, 
\; \text{ and } \; 
{\norm{\tilde{M}_o(:,i)}_1 Y(i,j) \leq \norm{\tilde{M}_o(:,j)}_1 Y(i,i) \; \forall \, i,j}   
\} , 
\end{equation} 
is equivalent to the LP~\eqref{gillisLP}. 
  This shows that the LP~\eqref{gillisLP} looks for an approximation $\tilde{M}_oY$ of $\tilde{M}_o$ 
   with small \emph{relative error}, which is in general not desirable in practice. 
   For example, a zero column to which some noise is added will have to be approximated rather well, while it does not bring any valuable information. Similarly, the columns of $M$ with large norms will be given relatively less importance while they typically contain a more reliable information (e.g., in document data sets, they correspond to longer documents).

It is now easy to modify the LP~\eqref{gillisLP1} to handle other noise models.  
For example, if the noise added to each column of the input data matrix is independent of its norm, then one should rather use the following LP trying to find an approximation $\tilde{M}_oY$ of  $\tilde{M}_o$ with small \emph{absolute error}:  
\begin{align} 
\min_{Y \in  \mathcal{Y}} & \; \; p^T \diag(Y)  
\quad \text{ such that } \quad \norm{\tilde{M}_o - \tilde{M}_o Y}_1 \leq \rho \epsilon.  \label{rechtgen}  
\end{align}

  \begin{remark}[Other noise models] 
Considering other noise models depending on the problem at hand is also possible: one has to replace the constraint $\norm{\tilde{M}_o - \tilde{M}_o Y}_{1} \leq \rho \epsilon$ with another appropriate constraint. For example, using any $\ell_q$-norm with $q \geq 1$ leads to efficiently solvable convex optimization programs \cite{FT04}, that is, using 
 \[
 \norm{\tilde{M}_o(:,j) - \tilde{M}_o Y(:,j)}_{q} \leq \rho \epsilon, \quad \text{ for all $j$}. 
 \] 
 Another possibility is to assume that the noise is distributed among all the entries of the input matrix independently 
and one could use instead  
 $\sqrt[q]{\sum_{i,j} \left( \tilde{M}_o - \tilde{M}_o Y \right)_{ij}^q}  \leq \rho \epsilon$, e.g., $\norm{\tilde{M}_o - \tilde{M}_o Y}_{F} \leq \rho \epsilon$ for Gaussian noise (where $||.||_F$ is the Frobenius norm of a matrix with $q = 2$). 
 \end{remark}

  \section{Numerical Experiments} \label{xp}

  In this section, we present some numerical experiments in which we compare our new LP model~\eqref{rechtgen} with Hottopixx and two other state-of-the-art methods.  First we describe a practical twist to Algorithm~\ref{cluster}, which we routinely apply in the experiments to LP-based solutions.

  \subsection{Post-Processing of LP solutions}
  
Recall that the LP-based algorithms return a nonnegative matrix $X$ whose diagonal entries indicate the importance of the corresponding columns of the input data matrix $\tilde{M}$. 
  As explained earlier, there are several ways to extract $r$ columns from $\tilde{M}$ using this information, the simplest being to select the columns corresponding to the $r$ largest diagonal entries of $X$ \cite{BRRT12}. 
  Another approach is to take into account the distances between the columns of $\tilde{M}$ and cluster them accordingly; see Algorithm~\ref{cluster}. 
  In our experiments we have not observed that one method dominates the other (although in theory, when the noise level is sufficiently small, Algorithm~\ref{cluster} is more robust; see \cite{G12b}). Therefore, the strategy we employ in the experiments below selects the best solution out of the two post-processing strategies based on the residual error, see Algorithm~\ref{hybpost}. 
       \algsetup{indent=2em}
\begin{algorithm}[ht!]
\caption{Hybrid Post-Processing for LP-based Near-Separable NMF Algorithms \label{hybpost}}
\begin{algorithmic}[1] 
\REQUIRE A matrix $M \in \mathbb{R}^{m \times n}$, 
a factorization rank $r$, a noise level $\epsilon$, and a vector of weight $x \in \mathbb{R}^{n}_+$. 
\ENSURE An index set $\mathcal{K}$ such that $\min_{H \geq 0} \norm{M-M(:,\mathcal{K})H}_F$ is small.  \medskip 

\STATE \emph{\% Greedy approach}
\STATE $\mathcal{K}_1$ is the set of the $r$ largest indices of $x$; 

\STATE \emph{\% Clustering using Algorithm~\ref{cluster}}
\STATE $\mathcal{K}_2$ = Algorithm~\ref{cluster}$\left( \tilde{M}, x, \epsilon, r \right)$; 

\STATE \emph{\% Select the better of the two}
\STATE $\mathcal{K} = \argmin_{\mathcal{R} \in \{\mathcal{K}_1, \mathcal{K}_2\}} \min_{H \geq 0} \norm{M-M(:,\mathcal{R})H}_F^2$;
\end{algorithmic} 
\end{algorithm}

  \subsection{Algorithms} 
  \label{sec:numexp_algs}

  In this section, we compare the following near-separable NMF algorithms: 
  \begin{enumerate}
  
  \item \textbf{Hottopixx} \cite{BRRT12}. Given the noise level $\norm{N}_1$ and the factorization rank $r$, it computes the optimal solution $X^*$ of the LP~\eqref{rechtLP} (where the input matrix $\tilde{M}$ has to be normalized) and returns the indices obtained using Algorithm~\ref{hybpost}. The vector $p$ in the objective function was randomly generated using the \texttt{randn} function of Matlab. 
  The algorithm of \cite{AGKM11} was shown to perform worse than Hottopixx \cite{BRRT12} hence we do not include it here (moreover, it requires an additional parameter $\alpha$ related to the conditioning of $W$ which is difficult to estimate in practice). 
  
  \item \textbf{SPA} \cite{MC01}. The successive projection algorithm (SPA) extracts recursively $r$ columns of the input normalized matrix $\tilde{M}$ as follows: at each step, it selects the column with maximum $\ell_2$ norm, and then projects all the columns of $\tilde{M}$ on the orthogonal complement of the extracted column. This algorithm was proved to be robust to noise \cite{GV12}. 
  (Note that there exist variants where, at each step, the column is selected according to other criteria, e.g., any $\ell_p$ norm with $1 < p < +\infty$. This particular version of the algorithm using $\ell_2$ norm actually dates back from modified Gram-Schmidt with column pivoting, see \cite{GV12} and the references therein.)   
  SPA was shown to perform significantly better on several synthetic data sets than Hottopixx and several state-of-the-art algorithms from the hyperspectral image community \cite{GV12} (these algorithms are based on the pure-pixel assumption which is equivalent to the separability assumption, see Introduction).

  \item \textbf{XRAY} \cite{KSK12}. In \cite{KSK12}, several fast conical hull algorithms are proposed. 
  We use in this paper the variant referred to as $max$, because it performs in average the best on synthetic data sets.  
  Similarly as SPA, it recursively extracts $r$ columns of the input unnormalized matrix $\tilde{M}_o$:  at each step, it selects a column of $\tilde{M}_o$ corresponding to an extreme ray of the cone generated by the columns of  $\tilde{M}_o$, and then projects all the columns of $\tilde{M}_o$ on the cone generated by the columns of $\tilde{M}_o$ extracted so far.  
  XRAY was shown to perform much better than Hottopixx and similarly as SPA on synthetic data sets 
  (while performing better than both for topic identification in document data sets as it does not require column normalization).  However, it is not known whether XRAY is robust to noise. 
   
  \item \textbf{LP~\eqref{rechtgen}} with $\rho = 1,2$.  Given the noise level $\norm{N}_1$, it computes the optimal solution $X^*$ of the LP~\eqref{rechtgen} and returns the indices obtained with the post-processing described in Algorithm~\ref{hybpost}. (Note that we have also tried $\rho = \frac{1}{2}$ which performs better than $\rho = 2$ but slightly worse than $\rho = 1$ in average hence we do not display these results here.)

  \end{enumerate}

  Table~\ref{comptable} gives the following information for the different algorithms: computational cost, memory requirement, parameters and if column normalization of the input matrix is necessary. 
  \begin{table}[ht!] 
\begin{center} 
\begin{tabular}{|c||c|c|c|c|}
\hline 
  & Flops    & Memory    &  Parameters &  Normalization     \\  \hline  \hline 
Hottopixx  & $\Omega\left(m n^2\right)$ &  $\mathcal{O}\left(mn + n^2\right)$ & $\norm{N}_1$, $r$ & Yes \\ [5pt] 
SPA  & 2mnr + $\mathcal{O}\left(m r^2\right)$ &	   $\mathcal{O}\left(mn\right)$  & $r$ & Yes \\ [5pt] 
XRAY   & $\mathcal{O}\left(mnr\right)$    & $\mathcal{O}\left(mn\right)$  & $r$ & No \\ [5pt] 
LP~\eqref{rechtgen} 
 & $\Omega\left(m n^2\right)$  &  $\mathcal{O}\left(mn + n^2\right)$ & $\norm{N}_1$ & No \\ [5pt] 
\hline  
\end{tabular} 
\caption{Comparison of robust algorithms for near-separable NMF for a dense $m$-by-$n$ input matrix.}  
\label{comptable}
\end{center}
\end{table}  
  
    The LP have been solved using the IBM ILOG CPLEX Optimizer\footnote{Available for free at \url{http://www-01.ibm.com/software/integration/optimization/cplex-optimizer/} for academia.} on a standard Linux box. Because of the greater complexity of the LP-based approaches (formulating \eqref{rechtLP} and \eqref{rechtgen} as LP's requires $n^2 + mn$ variables), 
    the size of the input data matrices allowed on a standard machine is limited, roughly $mn^2 \sim 10^6$ (for example, on a two-core machine with 2.99GHz and 2GB of RAM, it already takes about one minute to process a 100-by-100 matrix using CPLEX). 
    In this paper, we mainly focus on the robustness performance of the different algorithms and first compare them on synthetic data sets. 
    We also compare them on the popular swimmer data set. 
		Comparison on large-scale real-world data sets would require dedicated implementations, such as the parallel first-order method proposed by \cite{BRRT12} for the LP~\eqref{rechtLP}, and is a topic for further research. 
  The code for all algorithms is available at \url{https://sites.google.com/site/nicolasgillis/code}.

  \subsection{Synthetic Data Sets}

With the algorithms above we have run a benchmark with certain synthetic
data sets particularly suited to investigate the robustness behaviour under
influence of noise.  In all experiments the problem dimensions are fixed to $m
= 50$, $n = 100$ and $r = 10$.  We conducted our experiments with six different
data models.  As we will describe next, the models differ in the way the factor
$H$ is constructed and the sparsity of the noise matrix~$N$.  Given a desired
noise level $\epsilon$, the noisy $r$-separable matrix $\tilde{M} = M + N = WH
+ N$ is generated as follows:

The entries of $W$ are drawn uniformly at random from the interval $[0,1]$
(using Matlab's \texttt{rand} function). Then each column of $W$ is 
normalized so that its entries sum to one. 

The first $r$ columns of $H$ are always taken as the identity matrix to satisfy
the separability assumption.  The remaining columns of $H$ and the noise matrix
$N$ are generated in two different ways (similar to~\cite{GV12}):

\begin{enumerate}

\item \emph{Dirichlet}. The remaining 90 columns of $H$ are generated according
    to a Dirichlet distribution whose $r$ parameters are chosen uniformly in
    $[0,1]$ (the Dirichlet distribution generates vectors on the boundary of
    the unit simplex so that $\norm{H(:,j)}_1 = 1$ for all~$j$).  Each entry of
    the noise matrix $N$ is generated following the normal distribution
    $\mathcal{N}(0,1)$ (using the \texttt{randn} function of Matlab). 

\item \emph{Middle Points}. The $\frac{r(r-1)}{2} = 45$ next columns of $H$
    resemble all possible equally weighted convex combinations of pairs from
    the $r$ leading columns of $H$.  This means that the corresponding 45
    columns of $M$ are the middle points of pairs of columns of $W$. The
    trailing 45 columns of $H$ are generated in the same way as above, using
    the Dirichlet distribution.  No noise is added to the first $r$ columns of
    $M$, that is, $N(:,1:r) = 0$, while all the other columns are moved toward
    the exterior of the convex hull of the columns of $W$ using 
    \begin{equation*}
        N(:,j) = M(:,j)-\bar{w},  \quad \text{ for } r+1 \leq j \leq n, 
    \end{equation*}
    where $\bar{w}$ is the average of the columns of $W$ (geometrically, this
    is the vertex centroid of the convex hull of the columns of $W$). 

\end{enumerate}

We combine these two choices for $H$ and $N$ with three options that control
the pattern density of~$N$, thus yielding a total of six different data
models:

\begin{enumerate}

\item \emph{Dense noise}. Leave the matrix $N$ untouched.

\item \emph{Sparse noise}.  Apply a mask to $N$ such that
    roughly 75\% of the entries are set to zero (using the \texttt{density}
    parameter of Matlab's \texttt{sprand} function).

\item \emph{Pointwise noise}. Keep only one randomly picked non-zero
    entry in each nonzero column of~$N$.

\end{enumerate}

Finally we scale the resulting matrix $N$ by a scalar such that $\norm{N}_1 =
\epsilon$.  In order to avoid a bias towards the natural ordering, the columns
of $\tilde{M}$ are permuted at random in a last step.

\subsubsection{Error Measures and Methodology}

Let $\mathcal{K}$ be the set of indices extracted by an algorithm.  In our
comparisons, we will use the following two error measures:

\begin{itemize}

\item \emph{Index recovery}: percentage of correctly extracted indices in
    $\mathcal{K}$ (recall that we \emph{know} the indices corresponding to the columns of
    $W$). 

\item \emph{$\ell_1$ residual norm}: We measure the relative $\ell_1$ residual by 
    \begin{equation}
        1 - \min_{H \ge 0}
        \frac{\norm{\tilde{M} - \tilde{M}(:, \mathcal{K}) H}_s}{\norm{\tilde{M}}_s}. \label{eqn:sumerr}
\end{equation}
\end{itemize}
Note that both measures are between zero and one, one being the best possible value, zero the worst.  

The aim of the experiments is to display the robustness of the algorithms from
Section~\ref{sec:numexp_algs} applied to the data sets described in the
previous section under increasing noise levels.  For each data model, we ran
all the algorithms on the same randomly generated data on a predefined range of
noise levels~$\epsilon$.  For each such noise level, 25 data sets were
generated and the two measures are averaged over this sample for each
algorithm.

\subsubsection{Results}

Figures~\ref{fig:dirichlet} and \ref{fig:middlepoints} display the results for the three experiments of ``Dirichlet'' and
``Middle Points'' types respectively.  For comparison purpose, we also display the value of
\eqref{eqn:sumerr} for the true column indices $\mathcal{K}$ of $W$ in $M$,
labeled ``true K'' in the plots. 
\begin{figure*}[hp]
\begin{center}
    \includegraphics[width=.49\textwidth]{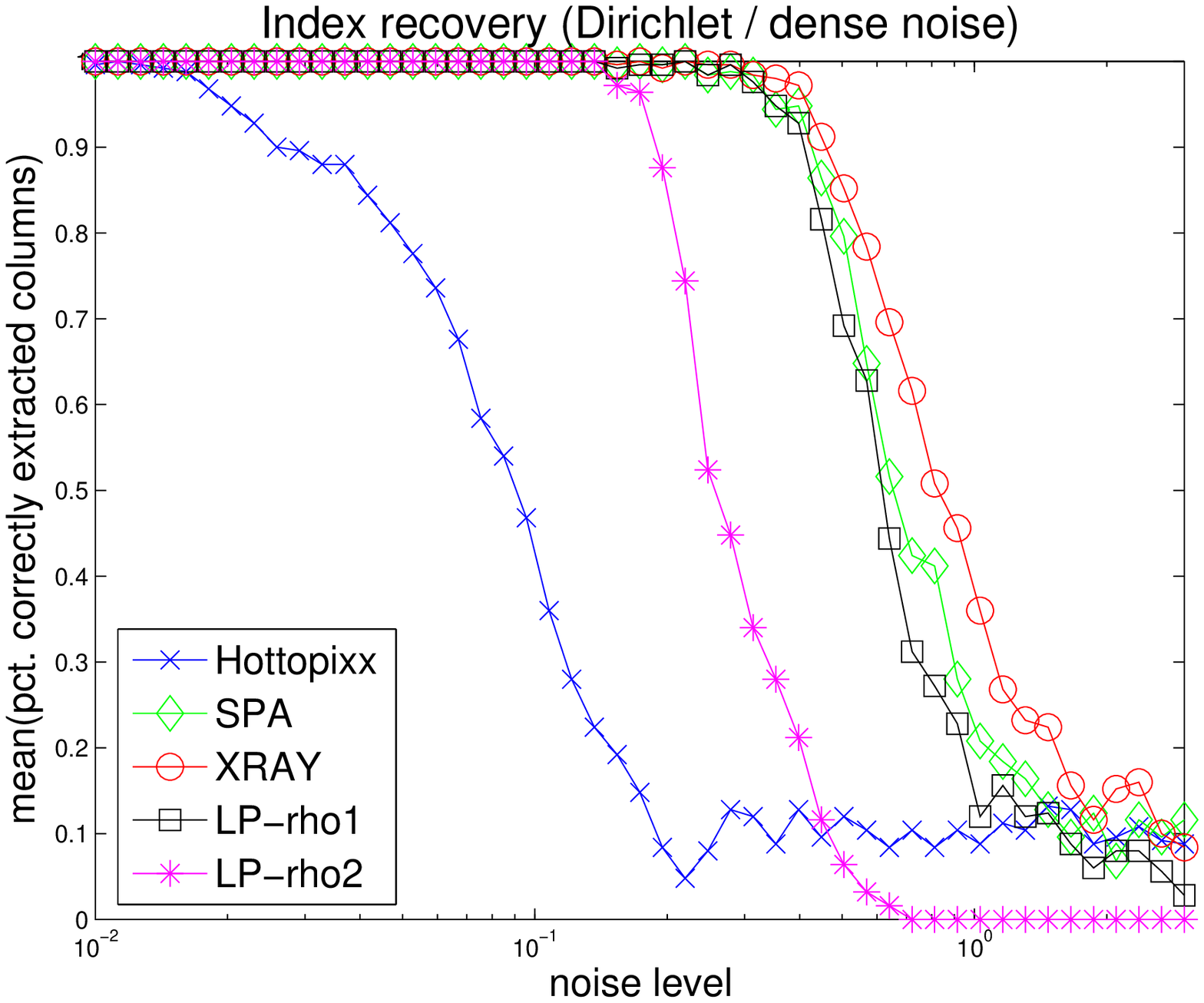}
    \hfill
    \includegraphics[width=.49\textwidth]{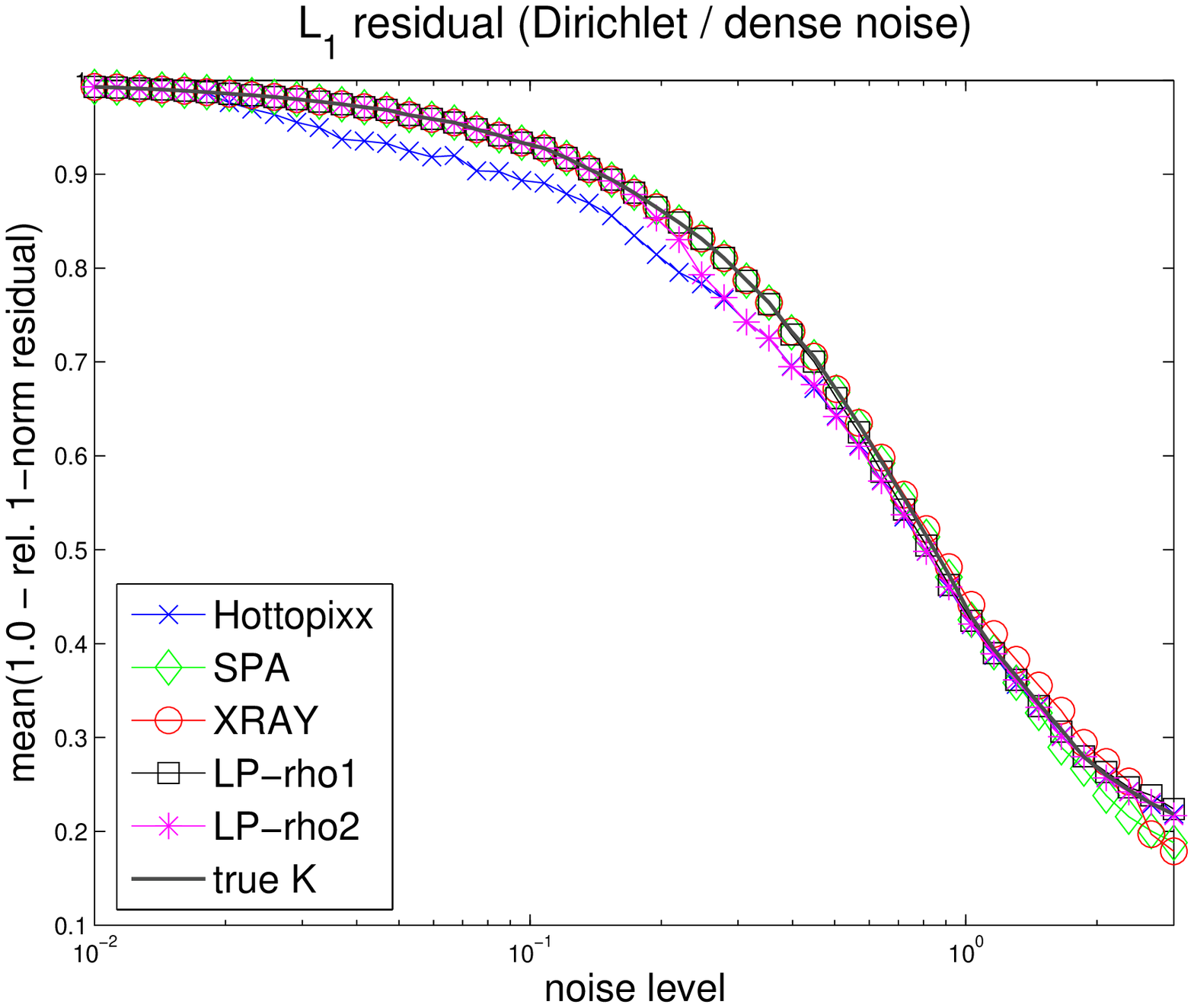}\\
    \vfill
    \includegraphics[width=.49\textwidth]{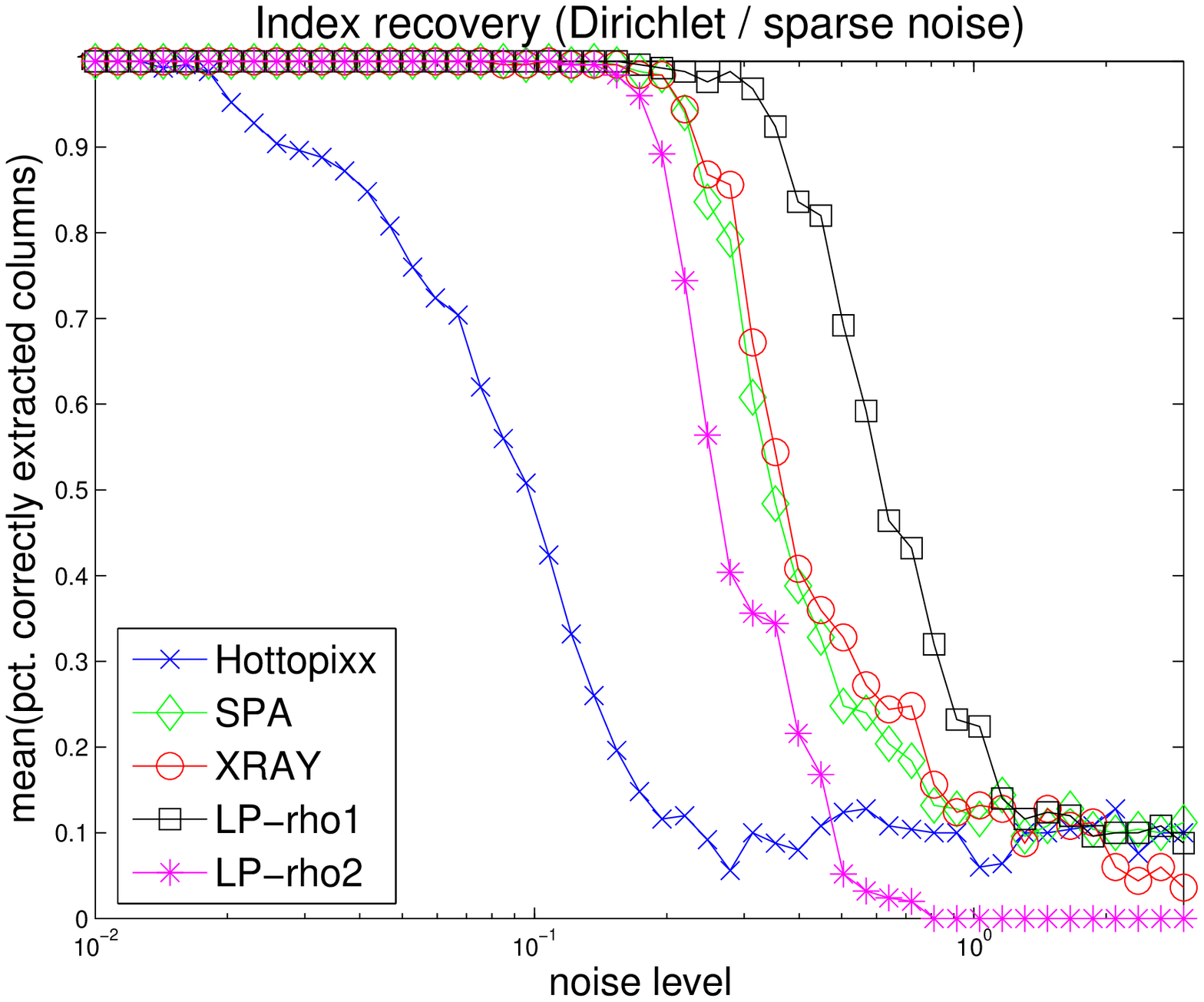}
    \hfill
    \includegraphics[width=.49\textwidth]{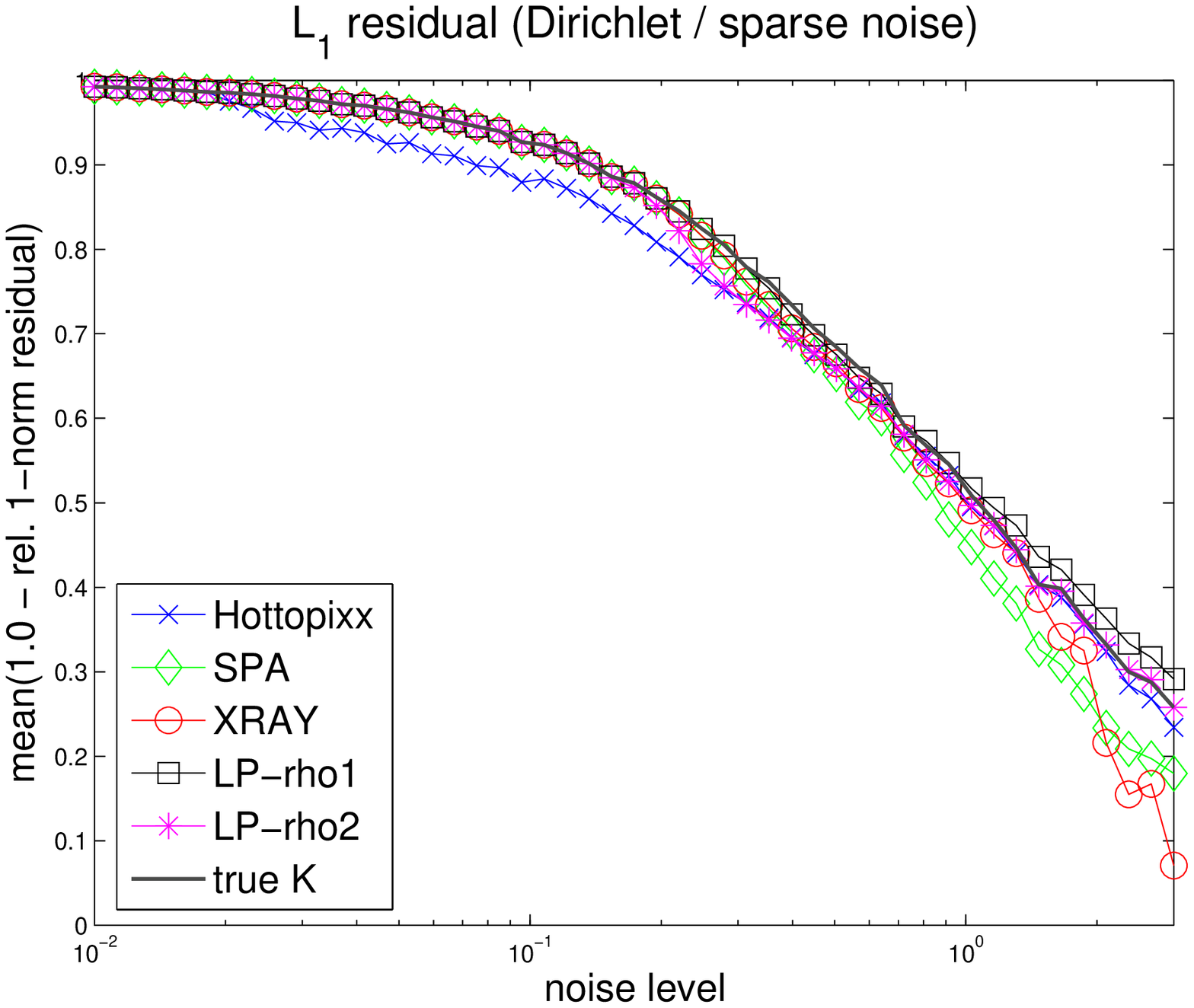}\\
    \vfill
    \includegraphics[width=.49\textwidth]{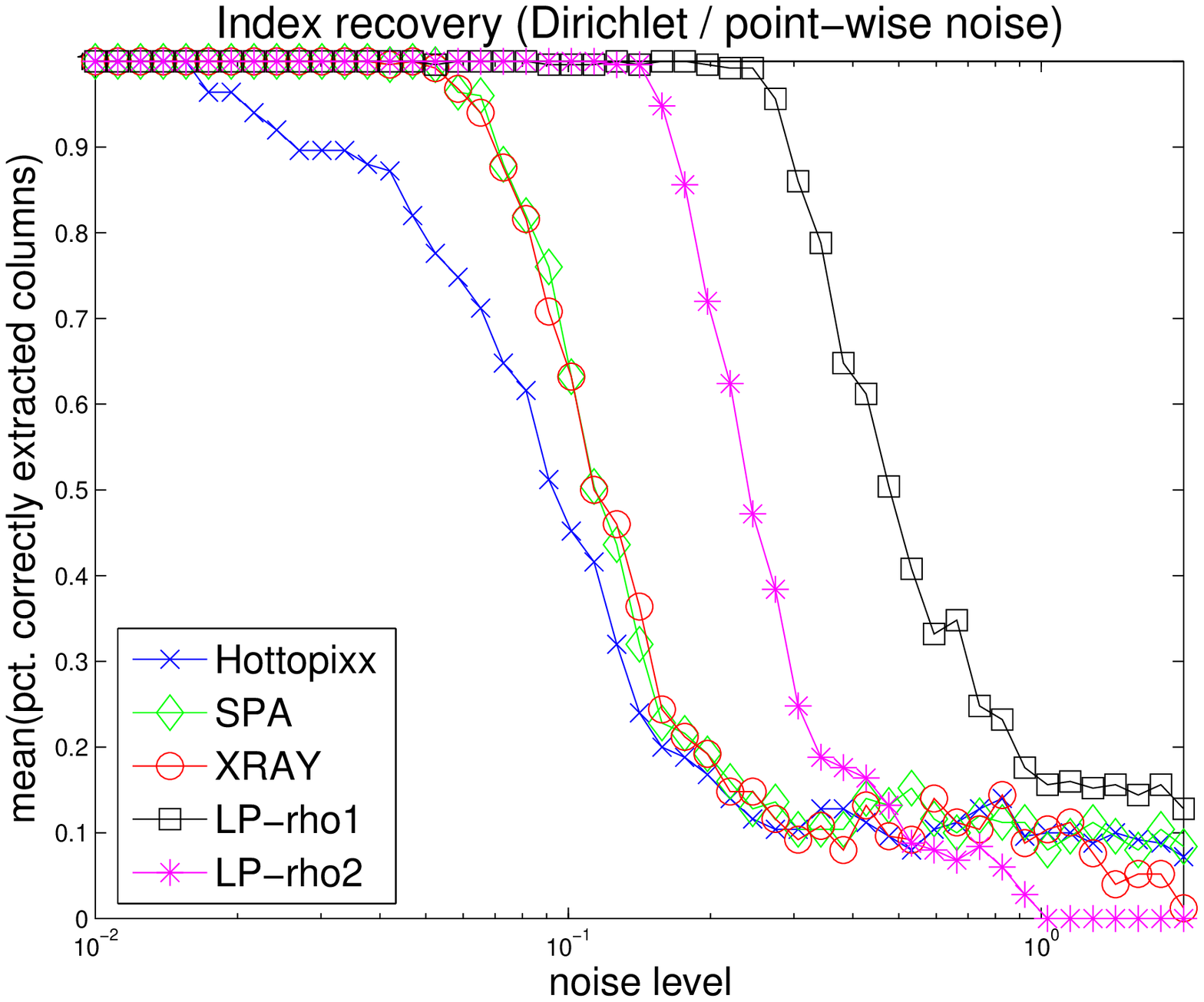}
    \hfill
    \includegraphics[width=.49\textwidth]{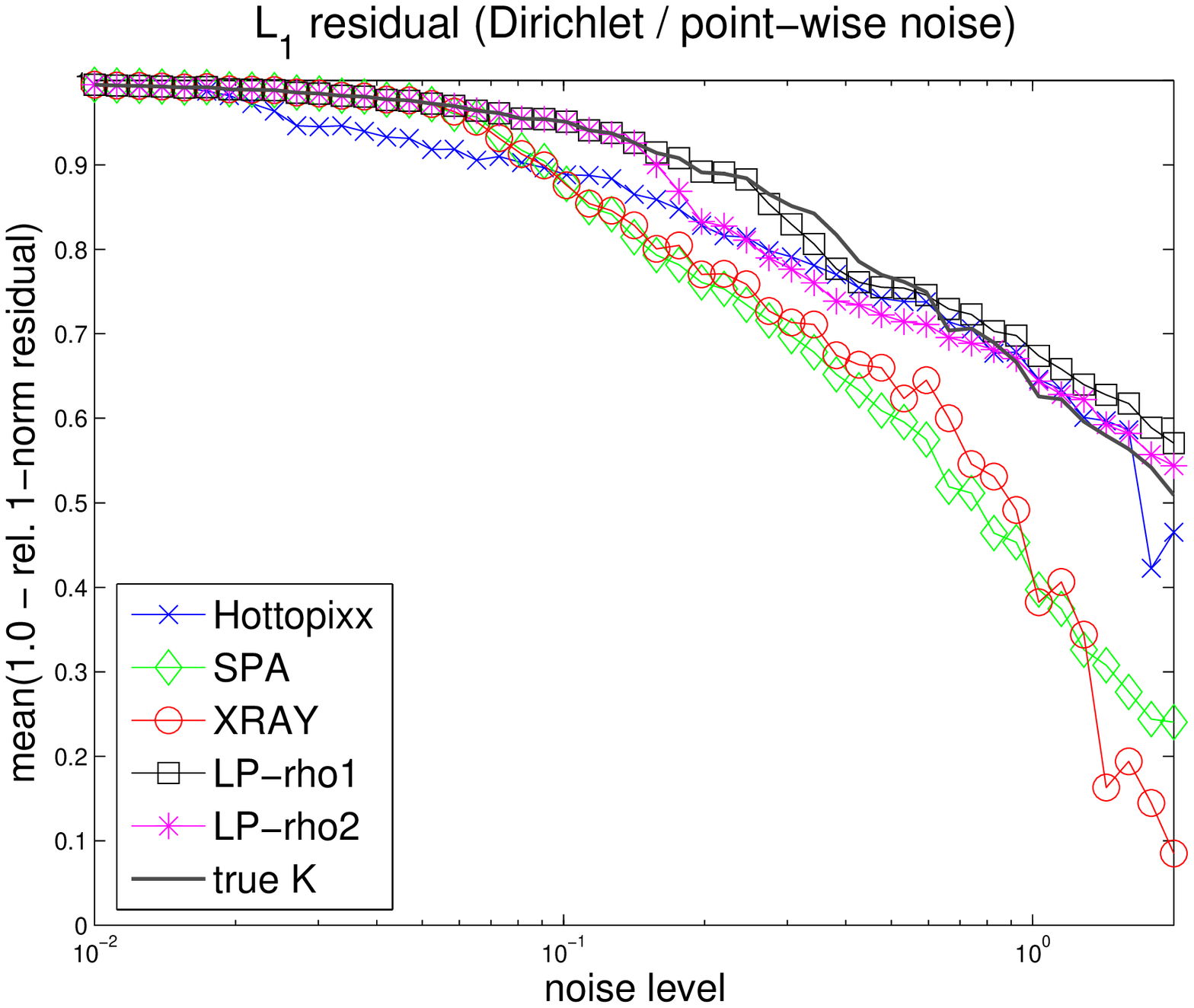}

    \caption{Comparison of near-separable NMF algorithms on ``Dirichlet'' type data
    sets.  From left to right: index recovery and $\ell_1$ residual. From top
    to bottom: dense noise, sparse noise and pointwise noise. }

    \label{fig:dirichlet}
\end{center}
\end{figure*} 
\begin{figure*}[hp]
\begin{center}
    \includegraphics[width=.49\textwidth]{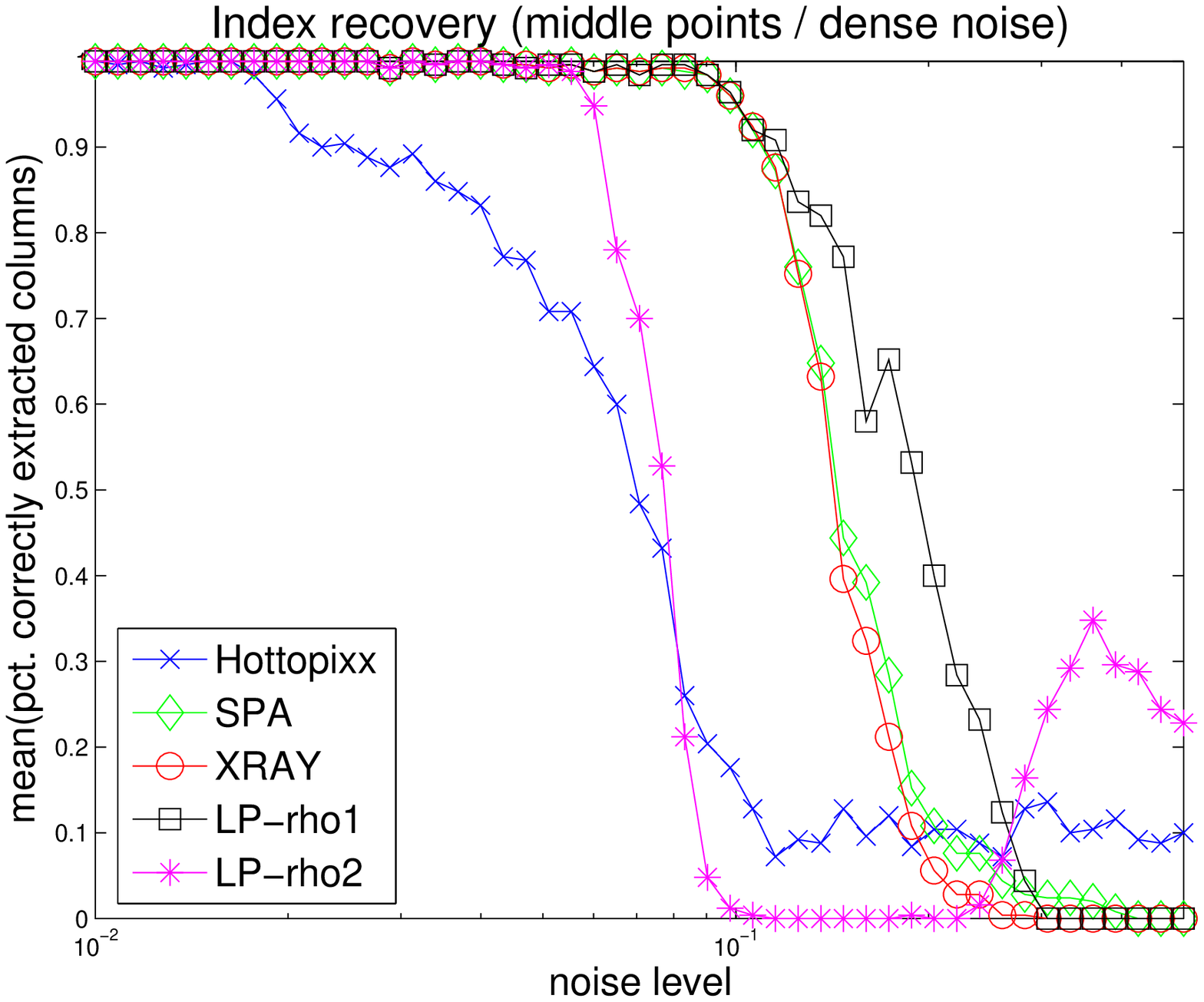}
    \hfill
    \includegraphics[width=.49\textwidth]{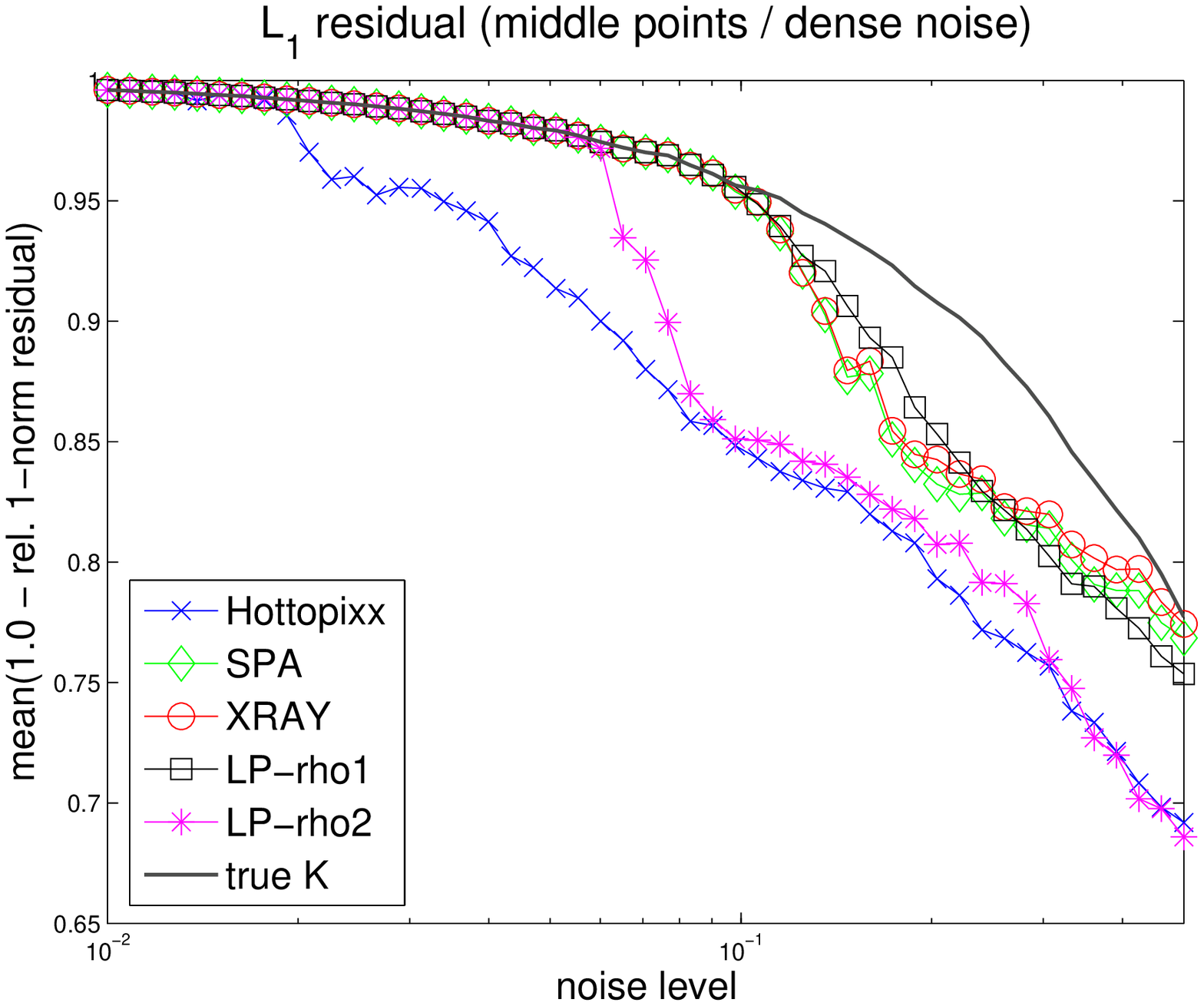}\\
    \vfill
    \includegraphics[width=.49\textwidth]{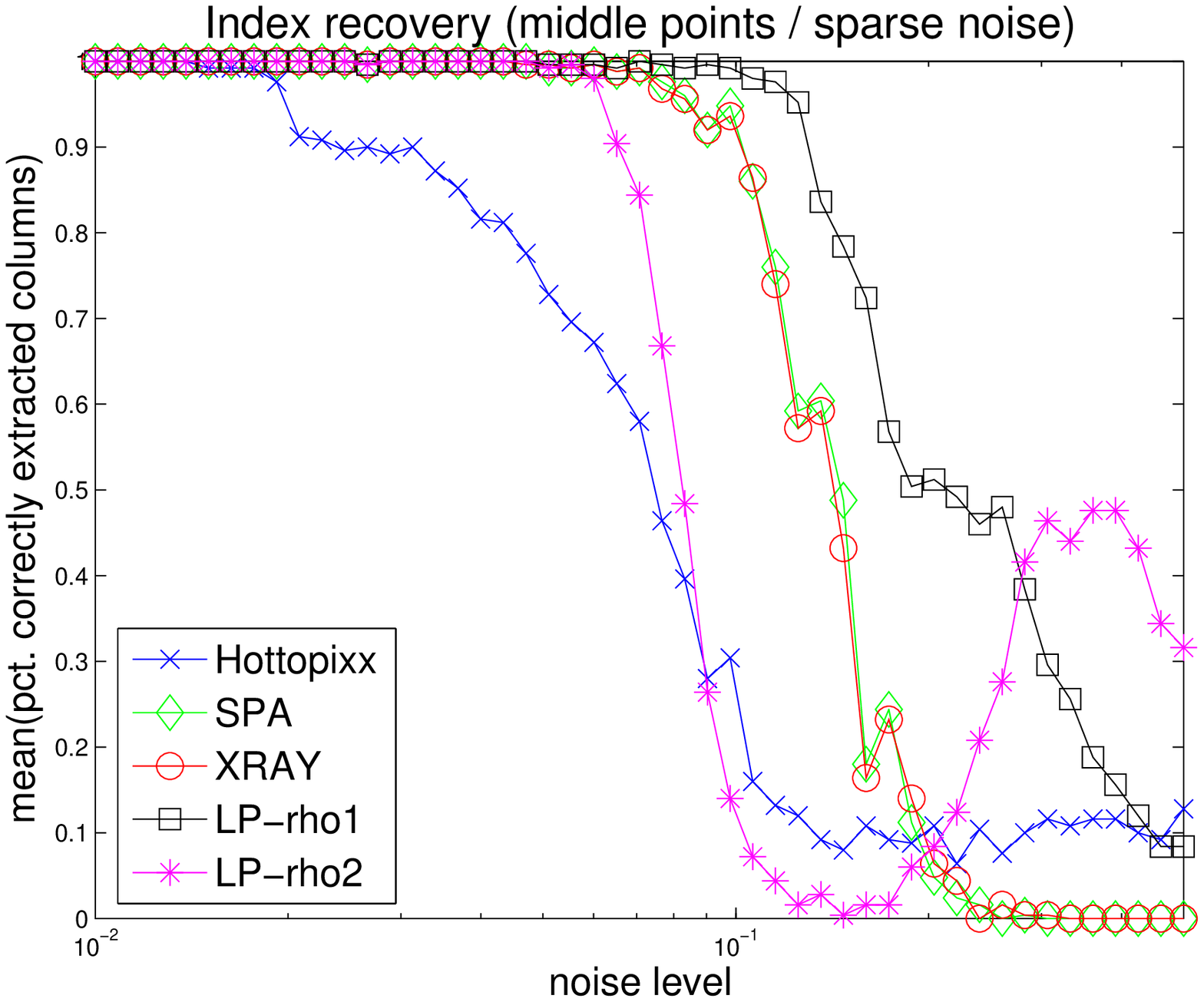}
    \hfill
    \includegraphics[width=.49\textwidth]{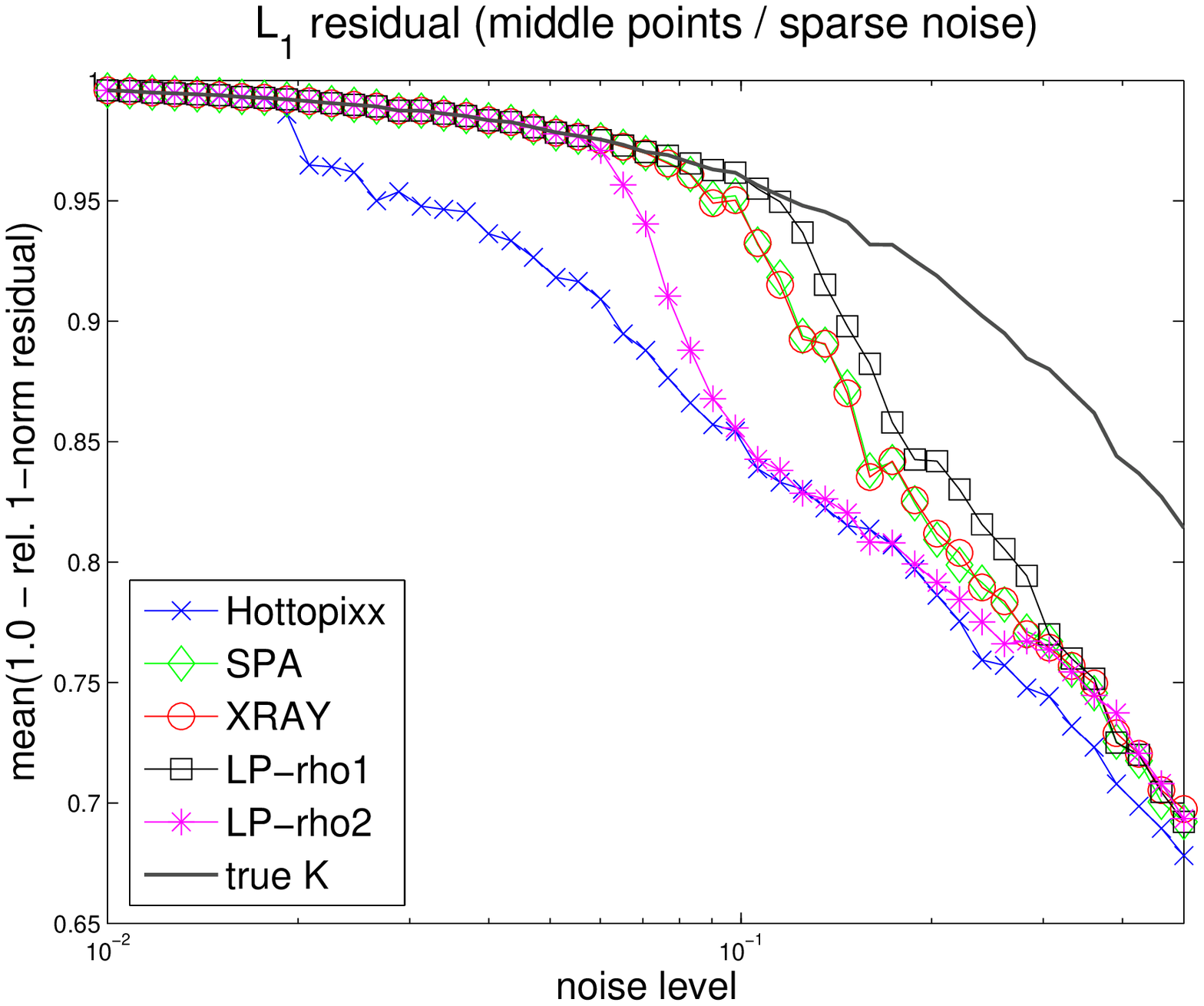}\\
    \vfill
    \includegraphics[width=.49\textwidth]{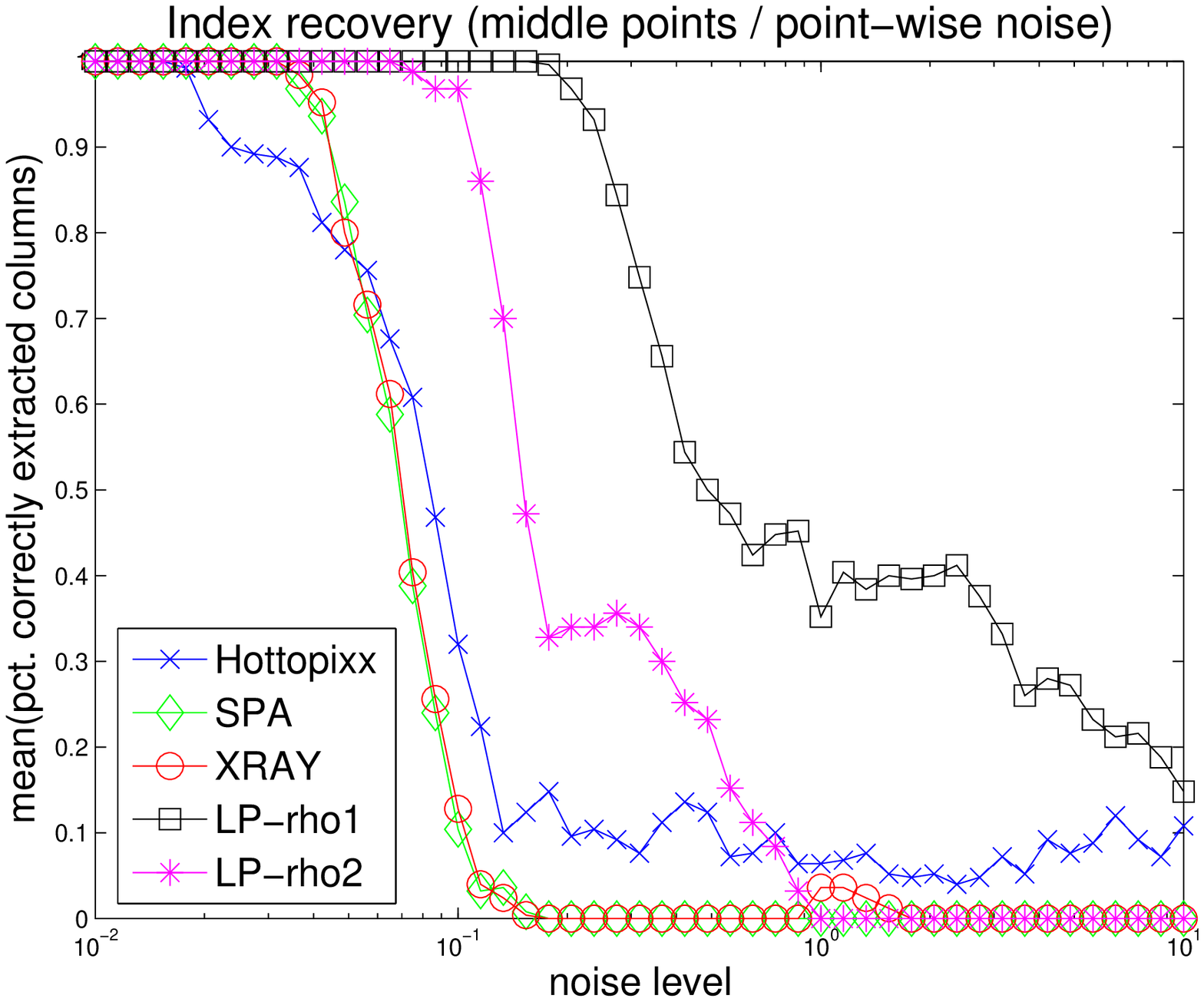}
    \hfill
    \includegraphics[width=.49\textwidth]{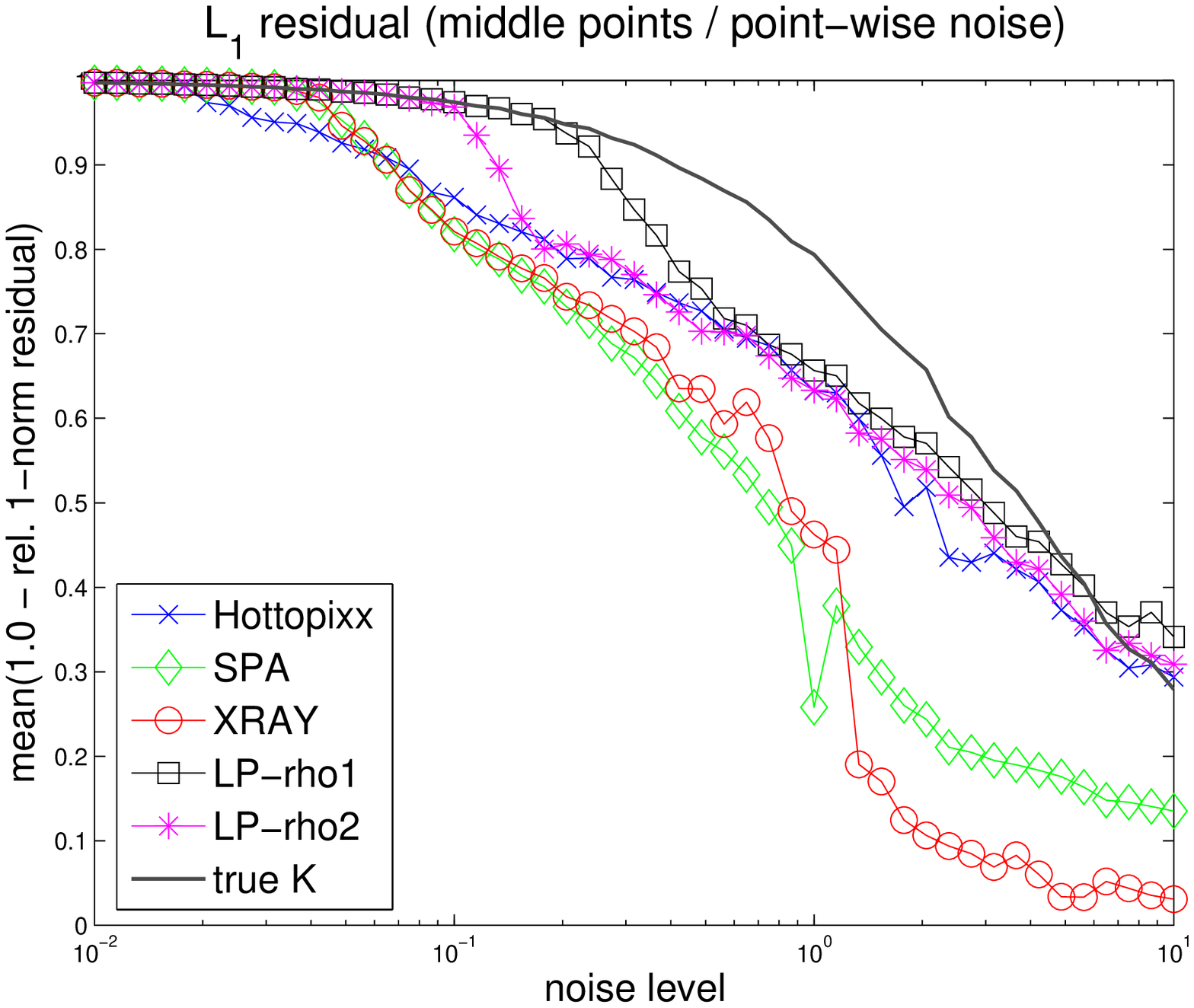}

    \caption{Comparison of near-separable NMF algorithms on ``Middle Points'' type
    data sets.  From left to right: index recovery and $\ell_1$ residual. From
    top to bottom: dense noise, sparse noise and pointwise noise. }

    \label{fig:middlepoints}
\end{center}
\end{figure*}  
    In all experiments, we observe that     
   \begin{itemize}

   \item The new LP model~\eqref{rechtgen} is significantly more robust to noise than Hottopixx, 
   which confirms our theoretical results; see Section~\ref{robdup}. 
   
   \item  The variant of LP~\eqref{rechtgen}  with $\rho = 2$ is less robust than with $\rho = 1$, 
   as suggested by our theoretical findings from Section~\ref{robdup}. 
   
   \item SPA and XRAY perform, in average, very similarly. 
   
   \end{itemize}
   Comparing the three best algorithms (that is, SPA, XRAY and LP~\eqref{rechtgen} with $\rho = 1$), we have that            
     \begin{itemize}
     
   \item  In case of ``dense'' noise, they give comparable results; although LP~\eqref{rechtgen} with $\rho = 1$ performs slightly worse for the ``Dirichlet'' type, and slightly better for the ``Middle Points'' type. 
   
   \item In case of ``sparse'' noise, LP~\eqref{rechtgen} with $\rho = 1$ performs consistently better then SPA and XRAY: for all noise levels, it identifies correctly more columns of $W$ and the corresponding NMF's have smaller $\ell_1$ residual norms.

   \item In case of ``pointwise'' noise, LP~\eqref{rechtgen} with $\rho = 1$ outperforms SPA and XRAY. In particular, for high noise level, it is able to extract correctly almost all columns of $W$ while SPA and XRAY can only extract a few for the ``Dirichlet'' type (performing as a guessing algorithm since they extract correctly only $r/n = 10\%$ of the columns of~$W$), or none for the ``Middle Points'' type. 
   
   Note that LP~\eqref{rechtgen} with $\rho = 2$ also performs consistently better then SPA and XRAY in case of ``pointwise'' noise.

   \end{itemize}

   \begin{remark}
   For the ``Middle Points'' experiments and for large noise levels, the middle points of the columns of $W$ become the vertices of the convex hull of the columns of $\tilde{M}$ (since they are perturbed toward the outside of the convex hull of the columns of $W$). 
   Hence, near-separable NMF algorithms should not extract any original column of $W$. 
   However, the index measure for LP~\eqref{rechtgen} with $\rho = 2$ increases for larger noise level (although the $\ell_1$ residual measure decreases); see Figure~\ref{fig:middlepoints}. 
   It is difficult to explain this behavior because the noise level is very high (close to 100\%)  hence 
   the separability assumption is far from being satisfied and it is not clear what the LP~\eqref{rechtgen} does.  
   \end{remark}

\begin{table}[t]
\begin{center}
\begin{tabular}{|c||c|c|c|c|c|c|}
\hline
           & D/dense  &   D/sparse  &  D/pw  & MP/dense &  MP/sparse & MP/pw \\ 
\hline \hline
Hottopixx  & 2.5      &   2.5       & 3.6    & 4.4      & 4.3        & 4.2\\
SPA        & $<$0.1     &   $<$0.1      & $<$0.1   & $<$0.1     & $<$0.1       & $<$0.1\\ 
XRAY       & $<$0.1     &   $<$0.1      & $<$0.1   & $<$0.1     & $<$0.1       & $<$0.1\\ 
LP \eqref{rechtgen}, $\rho = 1$ 
           & 20.5     &   34.1      & 39.0   & 52.5     & 88.1       & 41.4\\ 
LP \eqref{rechtgen}, $\rho = 2$ 
           & 10.5     &   12.3      & 16.0   & 32.5     & 56.9       & 27.4\\ 
\hline  
\end{tabular}
\caption{Average computational time in seconds for the different algorithms and data models. (D~stands for Dirichlet, MP for middle points, pw for pointwise.)} 
\label{ct}
\end{center}
\end{table}
Table~\ref{ct} gives the average computational time for a single application of
the algorithms to a data set.  As expected, the LP-based methods are
significantly slower than SPA and XRAY; designing faster solvers is definitely
an important topic for further research.  Note that the Hottopixx model can be
solved about ten times faster on average than the LP model~\eqref{rechtgen},
despite the only essential difference being the trace constraint $\tr(X) = r$.
It is difficult to explain this behaviour as the number of simplex iterations
or geometry of the central path cannot easily be set in relation to the
presence or absence of a particular constraint.

Table~\ref{tipping_point} displays the index recovery robustness:  For each
algorithm and data model, the maximum noise level $\norm{N}_1$ for which the
algorithm recovered on average at least 99\% of the indices corresponding to the columns of $W$.  In all
cases, the LP~\eqref{rechtgen} with $\rho=1$ is on par or better than all other
algorithms. 
\begin{table}[t]
\begin{center}
\begin{tabular}{|c||c|c|c|c|c|c|}
\hline
           & D/dense  &   D/sparse  &  D/pw  & MP/dense &  MP/sparse & MP/pw \\ 
\hline \hline
Hottopixx  & 0.014    &   0.018     & 0.016  & 0.016    & 0.018      & 0.015\\
SPA        & 0.220    &   0.154     & 0.052  & 0.077    & 0.071      & 0.032\\ 
XRAY       & \textbf{0.279}    &   0.154     & 0.052  & \textbf{0.083}    & 0.071      & 0.032\\ 
LP \eqref{rechtgen}, $\rho = 1$ 
           & \textbf{0.279}    &   \textbf{0.195}     & \textbf{0.197}  & \textbf{0.083}    & \textbf{0.098}      & \textbf{0.178}\\ 
LP \eqref{rechtgen}, $\rho = 2$ 
           & 0.137    &   0.121     & 0.141  & 0.055    & 0.055      & 0.075\\ 
\hline  
\end{tabular}
\caption{Index recovery robustnesss: Largest noise level $\norm{N}_1$ for which
an algorithm achieves almost perfect index recovery (that is, at least $99\%$
on average).} 
\label{tipping_point}
\end{center}
\end{table}

  \subsection{Swimmer Data Set}

	The swimmer data set is a widely used data set for benchmarking NMF algorithms \cite{DS03}. 
	It consists of 256 binary images of a body with four limbs which can be each in four different positions; see Figure~\ref{fig:swimmer}. 
		\begin{figure*}[hp]
\begin{center}
    \includegraphics[width=3cm]{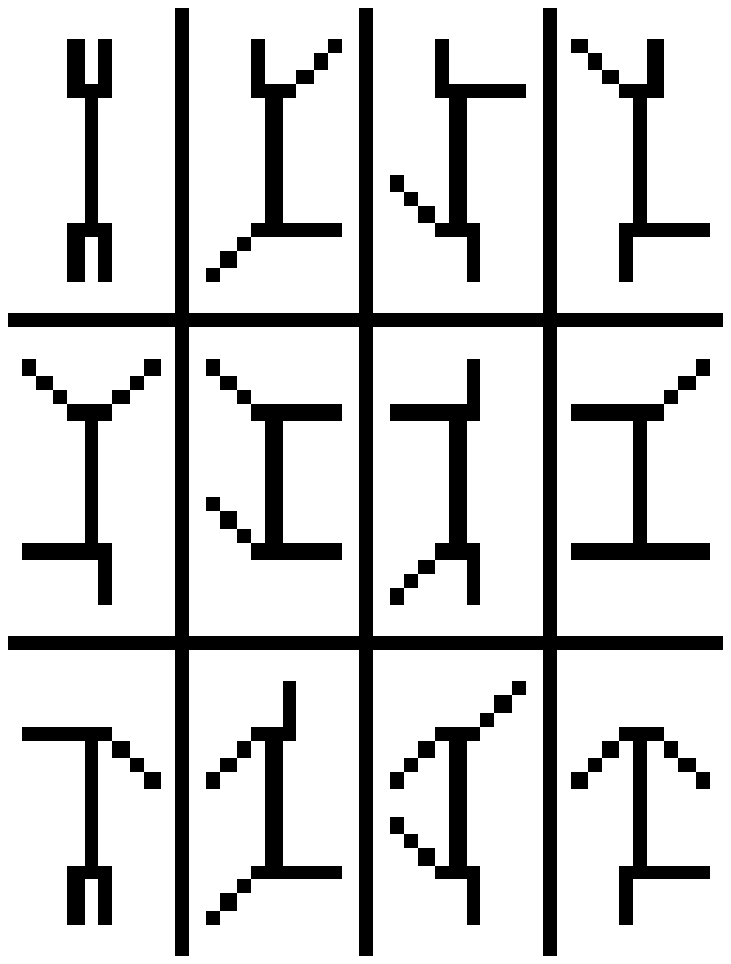}
    \caption{Sample images of the swimmer data set.}
    \label{fig:swimmer}
\end{center}
\end{figure*} 
	Let $M \in \{0,1\}^{256 \times 220}$ correspond to the swimmer data set where each row corresponds to an image, and each column to a pixel. It turns out that the matrix $M$ is 16-separable:  up to permutation, $M$ has the following form 
	\[
	M  \, =  \, W  \, \left[I_{16}, \, I_{16},  \, I_{16},  \, \frac{1}{4} E_{16 \times 14},   \, 0_{16 \times 158}\right] , 
	\]
	where $E_{m \times n}$ denotes the $m$-by-$n$ all-one matrix. 
	In fact, all the limbs are disjoint and contain three pixels (hence each column of $W$ is repeated three times), the body contains fourteen pixels and the remaining 158 background pixels do not contain any information. 
	
	\begin{remark}[Uniqueness of $H$]
	Note that the weights $\frac{1}{4} E_{16 \times 14}$ corresponding to the pixels belonging to the body are not unique. The reason is that the matrix $W$ is not full rank (in fact, $\rank(W) = 13$) implying that the convex hull of the columns of $W$ and the origin is not a simplex (that is, $r+1$ vertices in dimension $r$). 
	Therefore, the convex combination needed to reconstruct a point in the interior of that convex hull is not unique (such as a pixel belonging to the body in this example); see the discussion in \cite{G12}. 
	\end{remark}

	Let us compare the different algorithms on this data set:  
	\begin{itemize}
	
	\item \textbf{SPA}. Because the rank of the input matrix $M$ is equal to thirteen, the residual matrix becomes equal to zero after thirteen steps and SPA cannot extract more than thirteen indices hence it fails to decompose $M$.  
	
	\item \textbf{XRAY}. At the first step, the criterion used by XRAY to identify an extreme ray of the convex hull of the columns of $M$ is maximized by all non-zero columns of $M$ hence any of them can be extracted. 
	Since there are 48 pixels belonging to a limb and only 14 to the body, XRAY is more likely to extract a pixel on a limb (after which it is able to correctly decompose $M$).  However, if the first pixel extracted by XRAY is a pixel of the body then XRAY requires to be run with $r=17$ to achieve a perfect decomposition. 
	Therefore, XRAY succeeds on this example only with probability $\frac{48}{62}$ $\sim 77\%$ (given that XRAY picks a column at random among the one maximizing the criterion). We consider here a run where XRAY failed, otherwise it gives the same perfect decomposition as the new LP based approaches; see below.

	\item \textbf{Hottopixx}. 
	With $\epsilon = 0$ in the Hottopixx LP model~\eqref{rechtLP},  the columns of $W$ are correctly identified and Hottopixx performs perfectly. However, as soon as $\epsilon$ exceeds approximately 0.03, Hottopixx fails in most cases. 
	In particular, if $p$ is chosen such that its smallest entry does not correspond to a columns of $W$, then it always fails (see also the discussion in Example~\ref{ex1}). Even if $p$ is not chosen by an adversary but is randomly generated, this happens with high probability since most columns of $M$ do not correspond to a column of $W$.

	\item \textbf{LP \eqref{gillisLP1} with $\rho = 1$}. For $\epsilon$ up to approximately 0.97, the LP model~\eqref{gillisLP1} (that is, the new LP model based on relative error) idenfities correctly the columns of $W$ and decomposes $M$ perfectly. 
	
	\item \textbf{LP \eqref{rechtgen} with $\rho = 1$}. For $\epsilon$ up to approximately 60, (note that the $\ell_1$ norm of  the columns of $W$ is equal to 64), the LP model~\eqref{rechtgen} (that is, the new LP model based on absolute error) identifies correctly the columns of $W$ and decomposes $M$ perfectly. 
	
	\end{itemize}
	
	Figure~\ref{fig:swimmerweights} displays the optimal weights corresponding to the columns of $M$ extracted with the different algorithms (that is, the rows of the matrix $H^* = \argmin_{H \geq 0} ||M-M(:,\mathcal{K})H||_F$ where $\mathcal{K}$ is the index set extracted by a given algorithm): 
	the error for SPA is 20.8, for XRAY 12, for Hottopixx 12 and for the new LP models 0. 
		\begin{figure*}[hp]
\begin{center}
    \includegraphics[width=12cm]{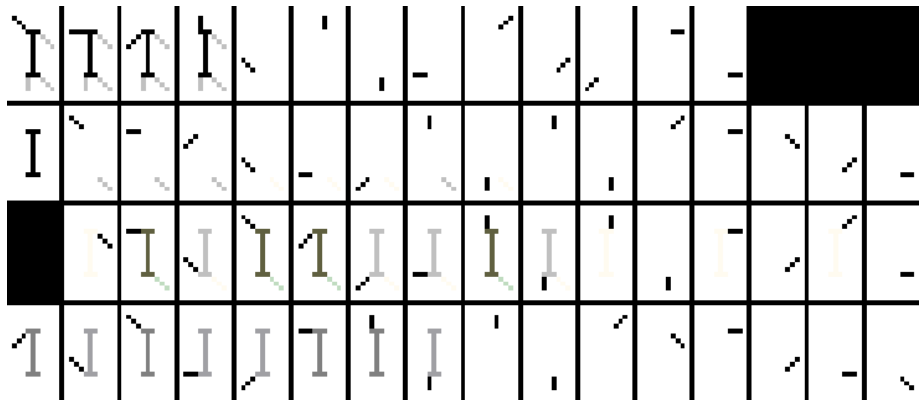}
    \caption{Weights corresponding to the extracted indices by the different algorithms. From to to bottom: SPA, XRAY, 
		Hottopixx ($\epsilon = 0.1$) and the new LP model~\eqref{gillisLP1} ($\epsilon = 0.1$).}
    \label{fig:swimmerweights}
\end{center}
\end{figure*} 
	Note that we used $\epsilon = 0.1$ for Hottopixx and the new LP models (a situation in which Hottopixx fails in most cases; see the discussion above --for the particular run shown in Figure~\ref{fig:swimmerweights}, Hottopixx extracts a background pixel corresponding to a zero column of $M$). 
	Note also that we do not display the result for the LP~\eqref{rechtgen} 
	because it gave an optimal solution similar to that of the LP~\eqref{gillisLP1}. 
	Finally, it is interesting to point out that the nonnegative rank of $M$ is equal to 16 hence the new LP models actually detect the nonnegative rank of $M$.

  \section{Conclusion and Further Work}  
\label{future}
  
  In this paper, we have proposed a new more practical and more robust LP model for near-separable NMF  which competes favorably with two state-of-the-art methods (outperforming them in some cases). 
It would be particularly interesting to investigate the following directions of research: 
  \begin{itemize} 
  
  \item Implementation and evaluation of an algorithm to solve \eqref{rechtgen} for large-sale real-world  problems. 
  
  \item Improvement of the theoretical bound on the noise level for Algorithm~\ref{postpro} to extract the right set of columns of the input data matrix in case duplicates and near duplicates are present in the data set (cf.\@ Section~\ref{rob}). 
  
  \item Design of practical and robust near-separable NMF algorithms. For example, would it be possible to design an algorithm as robust as our LP-based approach but computationally more effective (e.g., running in $\mathcal{O}(mnr)$ operations)?  
  
  \end{itemize}

	 \section*{Acknowledgments}
 
The authors would like to thank the reviewers for their feedback which helped improve the paper. \vspace{0.2cm}

  \appendix

  \section{Proof of Theorem~\ref{th1}} \label{appA}

  The next two lemmas are simple generalizations of Lemmas 2 \& 3 in \cite{G12b}. 
	Given any feasible solution $X$ of the the linear program \eqref{gillisLP}, the first one shows that 
	the $\ell_1$ norm of the error $M-MX$ with respect to the original noiseless data matrix is proportional to~$\epsilon$, that is, $\norm{\tilde{M} - \tilde{M} X}_1 \leq \mathcal{O}(\epsilon)$. 
The second one proves that the diagonal entries of $X$ corresponding to the columns of $W$ must be larger than $1-\mathcal{O}(\epsilon)$. 
	
\begin{lemma} \label{lem1}
Suppose $\tilde{M} = M+N$ where $\norm{M(:,j)}_1 = 1$ for all $j$ and $\norm{N}_{1} \leq \epsilon < 1$, and suppose $X$ is a feasible solution of \eqref{gillisLP}. Then, 
\[
\norm{X}_1 \leq  1 + \epsilon \left(\frac{\rho+2}{1-\epsilon} \right) 
\quad 
\text{ and } 
\quad 
\norm{M - M X}_1 \leq \epsilon \left(\frac{\rho+2}{1-\epsilon} \right). 
\]
\end{lemma}
\begin{proof}
First note that $\norm{\tilde{M}}_1 \le \norm{M}_1 +
\norm{N}_1 \le 1 + \epsilon$ and $\norm{MX}_1 = \norm{X}_1$. By the feasibility
of $X$ for \eqref{gillisLP}, 
\begin{equation*}
    \rho \epsilon 
    \ge \norm{\tilde{M} - \tilde{M} X}_1
    \ge \norm{\tilde{M} X}_1 - \norm{\tilde{M}}_1
    \ge \norm{MX}_1 - \norm{NX}_1 - (1+\epsilon)
    \ge \norm{X}_1 - \epsilon \norm{X}_1 - 1 - \epsilon,
\end{equation*}
hence $\norm{X}_1 \le 1 + \epsilon \left(\frac{\rho + 2}{1 - \epsilon}\right)$, implying
that $\norm{NX}_1 \le \norm{N}_1 \norm{X}_1 \le \epsilon \left(1 + \frac{(\rho
+2)\epsilon}{1 - \epsilon}\right)$.  Therefore
\begin{equation*}
    \rho \epsilon
    \ge \norm{\tilde{M} - \tilde{M} X}_1
    = \norm{M + N - (M+N)X}_1
    \ge \norm{M - MX}_1 - \epsilon - \epsilon
        \left(1 + \frac{(\rho +2)\epsilon}{1 - \epsilon}\right),
\end{equation*}
from which we obtain
$\norm{M - MX}_1
\le \epsilon \left( \rho + 2 + \frac{(\rho+2)\epsilon}{1 - \epsilon}\right)
= \epsilon \left( \frac{\rho + 2}{1 - \epsilon} \right)$
\end{proof}

\begin{lemma} \label{lem2}
Let $\tilde{M} = M+N$ where $||M(:,j)||_1 = 1$ for all $j$, admits a rank-$r$ separable factorization $WH$ with $W$ $\kappa$-robustly conical and $||N||_{1} \leq \epsilon < 1$, and has the form \eqref{sepM} with $\max_{i,j} H'_{ij} \leq \beta < 1$ and $W, H \geq 0$. Let also $X$ be any feasible solution of \eqref{gillisLP}, then 
\[
X(j,j) \geq  1 - \frac{2 \epsilon}{\kappa (1-\beta)} \left( \frac{\rho+2}{1-\epsilon} \right) 
\] 
for all $j$ such that $M(:,j) = W(:,k)$ for some $1 \leq k \leq r$. 
\end{lemma}
\begin{proof} 
The idea of the proof is the following: by assumption, each column of $W$ is isolated from the convex hull of the other columns of $M$. Therefore, to being able to approximate it up to error $\mathcal{O}(\epsilon)$, its corresponding diagonal entry must be large enough. 

Let $\mathcal{K}$ be the set of $r$ indices such that $M(:,\mathcal{K}) = W$. Let also $1 \leq k \leq r$ and denote $j = \mathcal{K}(k)$ so that $M(:,j) = W(:,k)$. 
By Lemma~\ref{lem1}, 
\begin{equation} \label{ubo}
||W(:,k) - WH X(:,j)||_1 \leq  \epsilon \left(\frac{\rho + 2}{1-\epsilon} \right). 
\end{equation} 
Since $H(k,j) = 1$, 
\begin{align*}
WH X(:,j) & = W(:,k) H(k,:) X(:,j)  + W(:,\mathcal{R}) H(\mathcal{R},:)X(:,j) \\ 
& = W(:,k) \Big( X(j,j) +  H(k,\mathcal{J})X(\mathcal{J},j) \Big)  + W(:,\mathcal{R}) y, 
\end{align*} 
where $\mathcal{R} = \{1,2,\dots,r\} \backslash \{k\}$, $\mathcal{J}= \{1,2,\dots,n\} \backslash \{j\}$ and $y = H(\mathcal{R},:)X(:,j) \geq 0$. We have  
\begin{equation} \label{etaa}
\eta = X(j,j) +  H(k,\mathcal{J}) X(\mathcal{J},j) \leq X(j,j) + \beta \left(1+ \frac{(\rho+2) \epsilon}{1-\epsilon} -X(j,j)\right), 
\end{equation}
since $||H(k,\mathcal{J})||_{\infty} \leq \beta$ and $||X(:,j)||_1 \leq 1+ \frac{(\rho+2) \epsilon}{1-\epsilon}$ (Lemma~\ref{lem1}). Hence 
\begin{equation} \label{eta}
||W(:,k) - WH X(:,j)||_1 \geq (1-\eta) \left\|W(:,k) - W(:,\mathcal{R}) \frac{y}{1-\eta}\right\|_1 \geq (1-\eta) \kappa. 
\end{equation}
Combining Equations~\eqref{ubo}, \eqref{etaa} and \eqref{eta},  we obtain 
\[
 1 - \left( X(j,j) + \beta \left( 1 + \frac{(\rho+2) \epsilon}{1-\epsilon} -X(j,j)\right) \right) \leq 
 \frac{\epsilon}{\kappa} \left(\frac{\rho + 2}{1-\epsilon} \right)
\]
which gives, using the fact that $\kappa, \beta \leq 1$, 
\begin{align*}
X(j,j) & \geq  1 - \frac{2 \epsilon}{\kappa (1-\beta)} \left( \frac{\rho+2}{1-\epsilon} \right) . 
\end{align*}
\end{proof}

If the diagonal entries corresponding to the columns of $W$ of a feasible solution $X$ of \eqref{gillisLP} are large, then the other diagonal entries will be small. In fact, the columns of $M$ are contained in the convex hull of the columns of $W$ hence can be well approximated with convex combinations of these columns.

\begin{lemma} \label{lem3}
  Let $\tilde{M} = M+N$ where $||M(:,j)||_1 = 1$ for all $j$, admits a rank-$r$ separable factorization $WH$ and $||N||_{1} \leq \epsilon$, and has the form \eqref{sepM}. Let $\mathcal{K}$ be the index set with $r$ elements such that $M(:,\mathcal{K}) = W$. Let also $X^*$ be an optimal solution of \eqref{gillisLP} such that  
\begin{equation} \label{condX}
X^*(k,k) \geq \gamma \quad \text{ for all $k \in  \mathcal{K}$},  
\end{equation}
where $0 \leq \gamma \leq 1$. Then, 
\[
X^*(j,j) 
\leq 1 - \min\left( \gamma, \frac{\rho}{2} \right) \quad  \text{ for all $j \notin  \mathcal{K}$}.  
\]
  \end{lemma} 
  \begin{proof} 
Let  $X$ be any feasible solution of \eqref{gillisLP} satisfying \eqref{condX}, and  $\alpha = \min\left( \gamma, \frac{\rho}{2} \right)$. 
Let us show that the $j$th column of $X$ for some $j \notin  \mathcal{K}$ can be modified as follows
\[
X(i,j) \leftarrow \left\{ \begin{array}{cc} 
1- \alpha & \text{ if } i = j,  \\
\alpha H(i,j) & \text{ if } i \in \mathcal{K},  \\ 
0 & \hspace{0.6cm} \text{ otherwise, } 
\end{array} \right.  
\] 
while keeping feasibility. 
First, $\alpha H(i,j) \leq \gamma \leq X(i,i)$ for all $i \in \mathcal{K}$ hence the condition $X(i,j) \leq X(i,i)$ for all $i,j$ is satisfied while, clearly, $0 \leq  X(i,i) \leq 1$ for all $i$. It remains to show that \mbox{$||\tilde{M}(:,j) - \tilde{M} X(:,j) ||_1$}$\leq \rho \epsilon$. 
By assumption, 
$M(:,j) = W H(:,j) = \alpha W H(:,j) + (1-\alpha) M(:,j)$ hence  
\begin{align*}
\tilde{M}(:,j) &  = \alpha \left( M(:,j) + N(:,j) \right) + (1-\alpha) \tilde{M}(:,j) \\ 
&  = \alpha \left( W  H(:,j) + N(:,j) \right) + (1-\alpha) \tilde{M}(:,j). 
\end{align*}
This gives  
\begin{align*}
||\tilde{M}(:,j) - \tilde{M} X(:,j) ||_1 
& = \alpha || M(:,j) + N(:,j) - (W+N(:,\mathcal{K})) H(:,j)  ||_1 
 \leq 2 \alpha \epsilon  \leq \rho \epsilon, 
\end{align*}
since the columns of $H$ sum to one, and $||N||_1 \leq \epsilon$. 
This result implies that any optimal solution $X^*$ satisfying \eqref{condX} must satisfy $X^*(j,j) \leq 1 - \alpha$, otherwise we could replace the $j$th column of $X^*$ using the construction above and obtain a strictly better solution since the vector $p$ in the objective function only has positive entries.  
  \end{proof}

We can now combine Lemmas~\ref{lem2} and \ref{lem3} to prove robustness of Algorithm~\ref{hot2} when there are no duplicates nor near duplicates of the columns of $W$ in the data set. \vspace{0.1cm}

  \begin{proof}[\textbf{Proof of Theorem~\ref{th1}}] 
  Let $X$ be an optimal solution of \eqref{gillisLP}.  
Let us first consider the case $\epsilon = 0$, which is particular because it allows duplicates of the columns of $W$ in the data set and the value of $\rho$ does not influence the analysis since $\rho \epsilon = 0$ for any $\rho > 0$. Let denote 
\[
\mathcal{K}_k = \{ j \ | \ M(:,j) = W(:,k) \}, 
\]
the set of indices whose corresponding column of $M$ is equal to the $k$th column of $W$. By assumption, $\kappa > 0$ hence for all $1 \leq k \leq r$ we have $W(:,k) \notin \cone(W(:,\bar{\mathcal{K}}))$ where $\bar{\mathcal{K}} = \{ 1,2,\dots,r\} \backslash \{k\}$. This implies that $\sum_{j\in \mathcal{K}_k} X(j,j) \geq 1$ for all $k$. Since we are minimizing a positive linear combination of the diagonal entries of $X$ and assigning a weight of one to each cluster $\mathcal{K}_k$ is feasible (see Equation~\ref{sepM}), we have $\sum_{j\in \mathcal{K}_k} X(j,j) = 1$. 
 Moreover, assigning all the weight to the index in $\mathcal{K}_k$ with the smallest entry in $p$ minimizes the objective function (and this index is unique since the entries of $p$ are distinct). 
 Finally, for all $1 \leq k \leq r$, there exists a unique $j$ such that $M(:,j) = W(:,k)$ and $X(j,j) = 1$ which gives the result for $\epsilon = 0$.

Otherwise $\epsilon > 0$ and $\beta < 1$, and the result follows from Lemmas~\ref{lem2} and \ref{lem3}: Let $\mathcal{K}$ be the set of $r$ indices such that $M(:,\mathcal{K}) = W$. 
By Lemma~\ref{lem2}, we have 
\[
X(k,k) \geq  1 - \frac{2 \epsilon}{\kappa (1-\beta)} \left( \frac{\rho + 2}{1-\epsilon} \right), \quad \text{ for all $k \in \mathcal{K}$}, 
\]
while, by Lemma~\ref{lem3}, 
\[
X(j,j) \leq \max\left( 1- \frac{\rho}{2}, \frac{2 \epsilon}{\kappa (1-\beta)} \left( \frac{\rho+2}{1-\epsilon} \right) \right), \quad \text{ for all $j \notin \mathcal{K}$}.  
\]
  Therefore, if   
  \[
  1 - \frac{2 \epsilon}{\kappa (1-\beta)} \left( \frac{\rho + 2}{1-\epsilon} \right) 
 \; > \;  f \;  \geq  \; 
 \max\left( 1- \frac{\rho}{2}, \frac{2 \epsilon}{\kappa (1-\beta)} \left( \frac{\rho + 2}{1-\epsilon} \right) \right), 
  \]
  where $f = 1 - \frac{\min(1,\rho)}{2} = \max\left( \frac{1}{2}, 1-\frac{\rho}{2}\right)$,  then Algorithm~\ref{hot2} extracts the $r$ indices corresponding to the columns of $W$. 
The above conditions are satisfied if  
\[
 \frac{2 \epsilon}{\kappa (1-\beta)} \left( \frac{\rho + 2}{1-\epsilon} \right) < \frac{\rho}{2} 
\quad \text{ and }  \quad 
\frac{2 \epsilon}{\kappa (1-\beta)} \left( \frac{\rho + 2}{1-\epsilon} \right)  < \frac{1}{2} , 
\] 
that is, 
$
\frac{\epsilon}{1-\epsilon} < \frac{\kappa (1-\beta) \min(1,\rho)}{4(\rho+2)}$. 
Taking 
\[
\epsilon \leq \frac{\kappa (1-\beta) \min(1,\rho)}{5(\rho+2)} < \frac{\kappa (1-\beta) \min(1,\rho)}{\rho+2} \frac{1-\epsilon}{4} 
\]
gives the results since $\epsilon \leq \frac{1}{5 (\rho+2)} < \frac{1}{10}$ for any $\rho > 0$ hence $\frac{1-\epsilon}{4} >  \frac{1}{5}$. 
  \end{proof}

  \section{Proof of Theorem~\ref{th1b}} \label{appAa}

  Theorem~\ref{th1b} can be proved using a particular construction.

 \begin{proof}[\textbf{Proof of Theorem~\ref{th1b}}]  
 Let us consider 
 \[
W = \left(
\begin{array}{c}
\frac{\kappa}{2} I_{r}  \\
(1-\frac{\kappa}{2})  e_r^T 
\end{array} 
\right), 
H = \left( 
I_r \quad \beta I_r + \frac{1-\beta}{r-1} (e_r e_r^T - I_r) \right), 
\text{ and } 
N = 0 , 
\] 
where $e_r \in \mathbb{R}^r$ is the all-ones vector,  $\frac{1}{r} \leq \beta <
1$ and $W$ is $\kappa$-robustly conical with $\kappa > 0$ \cite{G12b}. Define
$p = \left(\begin{smallmatrix}Ke_r\\e_r\end{smallmatrix}\right)$ for some large constant $K$ constant.   The matrix 
\[
X = \left( \begin{array}{cccc} 
 \left( 1 - \frac{\rho \epsilon}{\kappa (1-\beta)} \right) I_r 
 &  0  \\ 
  \frac{\rho \epsilon}{\kappa (1-\beta)}   I_r & I_r \\ 
  \end{array} \right) 
\]
 is a feasible solution of \eqref{gillisLP} for any $\epsilon \leq \frac{\kappa (1-\beta)}{\rho}$. In fact, for all $1 \leq j \leq r$, 
 \[
 ||M(:,j) - M X(:,j) ||_1 = \frac{\rho \epsilon}{\kappa (1-\beta)} ||M(:,j) -  M(:,j+r) ||_1 
  =  \rho \epsilon, 
 \]
 and it can be easily checked that $X$ satisfies the other constraints. 
 By Lemma 7 of \cite{G12b}, for $K$ sufficiently large, any optimal solution $X^*$ of \eqref{gillisLP} must satisfy 
 \[
\min_{1 \leq k \leq r} X^*(k,k) \leq \max_{1 \leq k \leq r} X(k,k) = 1 - \frac{\rho \epsilon}{\kappa (1-\beta)}, 
 \]
  (otherwise $p^T \diag(X^*) >  p^T \diag(X)$ for $K$ sufficiently large). 
 For the columns of $W$ to be extracted, one requires $X^*(k,k) > 1 - \frac{\min(1,\rho)}{2}$ for all $1 \leq k \leq r$ hence it is necessary that 
 \[
 1 - \frac{\rho \epsilon}{\kappa (1-\beta)} > 1 - \frac{\min(1,\rho)}{2} 
 \; \iff \; \epsilon < \frac{\kappa (1-\beta)}{2} \frac{\min(1,\rho)}{\rho}, 
 \]
 for Algorithm~\ref{hot2} to extract the first $r$ columns of $M$.  
 \end{proof}

  \section{Proof of Theorem~\ref{th2}} \label{appB}

 \begin{proof}[\textbf{Proof of Theorem~\ref{th2}}] 
The matrix $X^0$ from Equation~\eqref{sepM} is a feasible solution of \eqref{gillisLP}; in fact, 
  \[ 
  ||\tilde{M} - \tilde{M} X^0||_1 = ||M + N - (M+N) X^0||_1 \leq ||M - M X^0||_1 + ||N||_1 + ||N X^0||_1 \leq 2 \epsilon, 
  \]
  since $M = M X^0$, $||N||_1 \leq \epsilon$ and $||N X^0||_1 \leq ||N||_1 ||X^0||_1 \leq \epsilon$ as $||X^0||_1 = 1$. 
  Therefore, since $p = e$, any optimal solution $X^*$ of \eqref{gillisLP} satisfies 
  \[
 \tr(X^*) = p^T \diag(X^*) \leq p^T \diag(X^0) = r. 
  \]
The result then directly follows from Theorem 5 in \cite{G12b}. In fact, Algorithm~\ref{postpro}  is exactly the same as Algorithm 3  in \cite{G12b} except that the optimal solution of \eqref{gillisLP} is used instead of \eqref{rechtLP} while Theorem 5 from \cite{G12b} does not need the entries of $p$ to be distinct and only the condition $\tr(X) \leq r$ is necessary. 
Note that Theorem 5 in \cite{G12b} guarantees that there are $r$ disjoint clusters of columns of $\tilde{M}$ around each column of $W$ whose weight is strictly larger $\frac{r}{r+1}$. Therefore, the total weight is strictly larger than $r - \frac{r}{r+1} > r-1$ while it is at most  $r$ (since $\tr(X^*) \leq r$) implying that $r = \Big\lceil  \sum_{i=1}^n X^*(i,i) \Big\rceil$. 
\end{proof}

  \section{Proof of Theorem~\ref{outth}} \label{appout}

  The proof of Theorem~\ref{outth} works as follows: Let $X$ be a feasible solution of \eqref{gillisLP}. First, we show that  the diagonal entries of $X$ corresponding to the columns of $W$ and $T$ must be large enough (this follows from Theorem~\ref{th1}). Second, we show that the $\ell_1$ norm of the rows of $X$ corresponding to the columns of $W$ (resp.\@ $T$) must be sufficiently large (resp.\@ low) because the columns of $W$ (resp.\@ $T$) must be used (resp.\@ cannot be used) to reconstruct the other columns of $M$. \vspace{0.1cm}

 \begin{proof}[\textbf{Proof of Theorem~\ref{outth}}] 
 In case $\beta = 1$, $\epsilon = 0$ and the proof is similar to that of Theorem~\ref{th1};  the only difference is that the condition from Equation~\eqref{angle} has to be used to show that no weight can be assigned to off-diagonal entries of  the rows of an optimal solution of \eqref{gillisLP} corresponding to the columns of $T$. 
 Otherwise $\beta < 1$ and there are no duplicate nor near duplicate of the columns of $W$ in the data set. 
 
 Let assume without loss of generality that $\tilde{M}$ has the form 
 \[
 \tilde{M} = [T,W,WH'] + N, 
 \] 
 that is, the first $t$ columns correspond to $T$ and the $r$ next ones to $W$. Let then $X$ be an optimal solution of \eqref{gillisLP}. 

Since $[W,T]$ is $\kappa$-robustly conical, Theorem~\ref{th1} applies (as if the columns of $T$ were not outliers) and,  for all $1 \leq k \leq r + t$, 
 \[
 X(k,k) \geq  1 - \frac{8 \epsilon}{\kappa (1-\beta) (1-\epsilon)} \geq \frac{1}{2}, 
 \]
 while $X(j,j) \leq \frac{8 \epsilon}{\kappa (1-\beta) (1-\epsilon)}  \leq \frac{1}{2}$ for all $j > r + t$, since $\epsilon \leq \frac{\nu (1-\beta)}{20 (n-1)}$ where $\nu = \min(\kappa,\eta,\delta)$. 
 Therefore, only the first $r+t$ indices can potentially be extracted by Algorithm~\ref{hotout}.  
 It remains to bound above (resp.\@ below) the off-diagonal entries of the rows of $X$ corresponding to $T$ (resp.\@ $W$).

 By Lemma~\ref{lem2} (see also \cite[Lemma 2]{G12b}), we have for all $1 \leq j \leq n$ 
 \[
 || {M}(:,j) - {M} X(:,j)||_1 
 \leq \frac{4\epsilon}{1-\epsilon}  
 \quad \text{ and }  \quad 
 ||X(:,j)||_1 \leq 1 + \frac{4 \epsilon}{1-\epsilon} . 
 \]
 Using the fact that $[W,T]$ is $\kappa$-robustly conical,  for all $1 \leq k \leq t$, we have 
 \[
 ||T(:,k) - MX(:,k)||_1 \geq 
 (1 - X(k,k)) \min_{x \geq 0, y \geq 0} ||T(:,k) - T(:,\bar{\mathcal{K}})x - Wy||_1 \geq (1 - X(k,k)) \kappa, 
 \]
 implying that for all $1 \leq k \leq t$ 
  \[
 X(k,k) \geq  1 - \frac{4 \epsilon}{\kappa (1-\epsilon)}  \geq \frac{1}{2},  
 \] 
 since $\frac{4}{1-\epsilon} \leq 5$ because $\epsilon \leq \frac{1}{20}$. 
 Therefore, 
 \[
 \sum_{j \neq k} X(j,k) \leq ||X(:,k)||_1 - X(k,k)
 \leq \frac{4 \epsilon}{1-\epsilon} + \frac{4 \epsilon}{\kappa (1-\epsilon)}  
 \leq  \frac{8 \epsilon}{\kappa (1-\epsilon)}   , 
 \]
 as $\kappa, \epsilon \leq 1$.  
 Let $t+1 \leq j \leq n$ and $1 \leq k \leq t$, we have 
 \begin{align*}
  || {M}(:,j) - {M} X(:,j)||_1  
  & \geq 
  \min_{x} \min_{y \geq 0}  ||T(:,k) + T(:,\bar{\mathcal{K}})y - Wx||_1
   \geq \eta X(k,j), 
 \end{align*}
 see Equation~\eqref{angle}, which implies 
 $X(k,j) \leq  \frac{4\epsilon}{\eta (1-\epsilon)}$. 
 Hence, for all $1 \leq k \leq t$, we have 
\[
 \sum_{j \neq k} X(k,j) \leq 
  (t-1) \frac{8  \epsilon}{\kappa (1-\epsilon)}   + (n-r-t)  \frac{4\epsilon}{\eta (1-\epsilon)} 
 \leq \frac{8 (n-1) \epsilon}{\nu (1-\epsilon)} \leq \frac{1}{2}.   
 \] 
 since $\nu = \min(\kappa,\eta,\delta)$. 
By assumption, for each $t+1 \leq k \leq t+r$,  there exists some $j$ satisfying 
 $M(:,j) = WH(:,j) \neq W(:,k)$ 
and 
\[
\min_{x \geq 0} ||M(:,j) - W(:,\bar{\mathcal{K}})x||_1 \geq \delta, \quad \text{where $\bar{\mathcal{K}} = \{1,2,\dots,r\} \backslash \{k\}$}, 
 \] 
 see Equation~\eqref{gc}. For $t+r <  j \leq n$, we have 
 $X(j,j) \leq \frac{8 \epsilon}{\kappa (1-\beta) (1-\epsilon)}$. Let us denote 
 \mbox{$\mu =  \frac{8 (n-r-t) \epsilon}{\kappa (1-\beta) (1-\epsilon)}$} which is an upper bound for the total weight that can be assigned to the columns of $M$ different from $W$ and $T$. 
 Then, using Equation~\eqref{gc}, we have 
 \begin{align*}
 \norm{M(:,j) - M X(:,j)}_1 
 & \geq 
 (1-\mu) \,  \min_{y \geq 0} 
 \left\|M(:,j) - \frac{1}{1- \mu} W X(t+1\text{:}r+t,j) - Ty\right\|_1 \\
  & \geq (1-\mu) \left(1-\frac{X(k,j)}{1-\mu}\right) \delta. 
\end{align*}
 This implies 
 \[
  \frac{X(k,j)}{1-\mu} \geq 1 - \frac{4 \epsilon}{\delta (1-\mu) (1-\epsilon)}  
  \]
  and 
 \begin{align*} 
 X(k,j) 
  & \geq 1 - \frac{8 (n-r-t) \epsilon}{\kappa (1-\beta) (1-\epsilon)}  - \frac{4 \epsilon}{\delta (1-\epsilon)}  \\
 &  \geq 1 - \frac{8 (n-1) \epsilon}{\nu (1-\beta) (1-\epsilon)}  
  \geq \frac{1}{2}, 
 \end{align*}
 since $\beta \leq 1$ and $\epsilon \leq \frac{\nu (1-\beta)}{20 (n-1)}$, and the proof is complete. 
\end{proof}

\vskip 0.2in

\bibliographystyle{spmpsci} 
\bibliography{Biography}

\end{document}